\documentclass[11pt]{article}

\usepackage[utf8]{inputenc} 
\usepackage{fullpage}
\usepackage[T1]{fontenc}    
\usepackage{url}            
\usepackage{booktabs}       
\usepackage{amsfonts}       
\usepackage{nicefrac}       
\usepackage{microtype}      
\usepackage{xcolor}         

\usepackage{amsmath}        
\usepackage{amsthm}
\usepackage{mathrsfs}
\usepackage{bbm}            
\usepackage{mathtools}
\usepackage{xcolor}
\usepackage{algorithm}
\usepackage{algorithmic}
\usepackage{thm-restate}
\usepackage{enumitem}

\usepackage[colorlinks, linkcolor=purple, filecolor = blue, citecolor = blue, urlcolor = olive]{hyperref}
\usepackage[capitalise]{cleveref}
\usepackage{natbib}

\allowdisplaybreaks

\theoremstyle{plain}
\newtheorem{theorem}{Theorem}[section]

\newtheorem{remark}[theorem]{Remark}

\theoremstyle{definition}

\newtheorem{assumption}{Assumption}

\newcommand{\defeq}{\coloneq}
\newcommand{\Reals}{\mathbb{R}} 
\newcommand{\ucb}{\mathrm{UCB}}
\newcommand{\lcb}{\mathrm{LCB}}
\newcommand{\ind}[1]{\mathbbm{1}{\{#1\}}}

\newcommand{\papertitle}{One Good Source is All You Need: Near-Optimal Regret for Bandits under Heterogeneous Noise} 

\newcommand{\algmab}{\texttt{SOAR}} 

\newcommand{\red}[1]{} 

\newcommand{\as}[1]{} 

\newcommand{\R}{{\mathbb R}}

\newcommand{\N}{{\mathbb N}}

\newcommand{\E}{{\mathbf E}}

\newcommand{\cD}{{\mathcal D}}

\newcommand{\cI}{{\mathcal I}}

\newcommand{\hsigma}{{\hat {\sigma}}}
\newcommand{\hmu}{{\hat {\mu}}}

\newcommand{\cG}{{\mathcal G}}

\newcommand{\cT}{{\mathcal T}}

\newcommand{\tO}{{\tilde O}}

\newcommand{\biggn}[1]{\bigg(#1\bigg)}

\newcommand{\biggsn}[1]{\bigg[#1\bigg]}

\newcommand{\abs}[1]{\Big|#1\Big|}
\newcommand{\algprec}{\textsc{Preprocess}}

\date{}

\title{\bfseries\papertitle}

\author{
	Amith Bhat Hosadurga Anand\\
	UIC\\
	\texttt{abhat69@uic.edu}
	\and
	Haipeng Luo\\
	USC\\
	\texttt{haipengl@usc.edu }
	\and
	Aadirupa Saha\\
	UIC\\
	\texttt{aadirupa@uic.edu}
}

\date{}

\begin{document}

\maketitle

\begin{abstract}
We study the $K$-armed Multiarmed Bandit (MAB) problem with $M$ heterogeneous data sources, each exhibiting unknown and distinct noise variances $\{\sigma_j^2\}_{j=1}^M$. The learner's objective is standard MAB regret minimization, with the additional complexity of adaptively selecting which data source to query from at each round.
We propose Source-Optimistic Adaptive Regret minimization (\algmab), a novel algorithm that quickly prunes high-variance sources using sharp variance-concentration bounds, followed by a \emph{`balanced min-max LCB-UCB approach'} that seamlessly integrates the parallel tasks of identifying the best arm and the optimal (minimum-variance) data source.
Our analysis shows that \algmab\ achieves an instance-dependent regret bound of
$
{\tO}\left({\sigma^*}^2\sum_{i=2}^K \frac{\log T}{\Delta_i} + \sqrt{K \sum_{j=1}^M \sigma_j^2}\right),
$
up to preprocessing costs depending only on problem parameters, where ${\sigma^*}^2 := \min_j \sigma_j^2$ is the minimum source variance and $\Delta_i$ denotes the suboptimality gap of the $i$-th arm.
This result is both surprising and promising as despite lacking prior knowledge of the minimum-variance source among $M$ alternatives, \algmab\ attains the optimal instance-dependent regret of standard single-source MAB with variance ${\sigma^*}^2$, while incurring only an small (and unavoidable) additive cost of $\tO(\sqrt{K \sum_{j=1}^M \sigma_j^2})$ towards the optimal (minimum variance) source identification.
Additionally, we show that \algmab\ achieves a minimax (instance-independent) regret of ${\tO}\left(\sigma^*\sqrt{KT} + \sqrt{K \sum_{j=1}^M \sigma_j^2}\right)$ for worst case problem instances.
Our theoretical bounds represent a significant improvement over some proposed baselines such as \emph{Uniform UCB} or \emph{Explore-then-Commit UCB}, which could potentially suffer regret scaling with $\sigma_{\max}^2$ in place of ${\sigma^*}^2$---a gap that can be arbitrarily large when $\sigma_{\max} \gg \sigma^*$.
Experiments on multiple synthetic problem instances and the real-world MovieLens\;25M dataset, demonstrating the superior performance of \algmab\ over all baselines.
Our work opens a new direction for adaptively leveraging heterogeneous data sources in bandit learning, extending the scope of standard MAB beyond traditional single-source settings, with applications to crowdsourcing with heterogeneous annotators, sensor selection in environmental monitoring, and multi-lab experimental design.

\end{abstract}

\section{Introduction}
Online learning and multi-armed bandits (MABs) constitute a central framework for sequential decision-making under uncertainty \citet{auer2002finite, CsabaNotes18} \as{use citep}.
In the classical stochastic MAB model, the learner repeatedly selects an arm and observes a noisy reward drawn from a fixed distribution associated with that arm.
A key implicit assumption in this literature is that feedback is generated from a single homogeneous data source.

However, in many real-world systems, observations may originate from \emph{multiple heterogeneous sources}, each exhibiting different and unknown levels of noise.
Such scenarios naturally arise in clinical trials conducted across hospitals with varying measurement reliability, recommender systems aggregating feedback from diverse user populations—for instance, movie recommendation platforms where user reviews and ratings vary widely in reliability—and online systems collecting results from multiple sensing or logging pipelines. 
In these settings, noisy feedback can significantly delay learning; therefore, successfully managing these multiple data sources while simultaneously identifying the most rewarding actions (arms) is critical for minimizing cumulative regret and maximizing overall system performance.
The motivation for studying this problem lies in its practical significance and inherent complexity. Decision-makers frequently face trade-offs between cheaper, noisy data sources and more expensive, reliable ones. 

\vspace{-5pt}
\paragraph{Related Work.} \as{I changed the citation package to natbib -- so please check the citations are reading correct}
This scenario naturally aligns with the literature on multi-fidelity learning and variance-adaptive bandits. Multi-fidelity bandits address scenarios in which learners can query information at different levels of fidelity, often at the expense of cost and accuracy. \cite{kandasamy2016multi} introduced the idea of leveraging inexpensive, low-fidelity evaluations to accelerate optimization with expensive, high-fidelity sources. Further developments (Song et al., 2019 \cite{song2019general} and Poloczek et al., 2017 \citep{poloczek2017multi}) \as{remove double texts and just use cite here} explored the use of Gaussian process models with Takeno et al., 2020  \citep{takeno2020multi} also extending to batched multi-fidelity queries. \\While related, our setting notably differs in that we focus explicitly on adaptive exploration across multiple sources with differing noise variances, rather than on explicit cost-accuracy trade-offs. \\ 

 \as{please link the table -- this is generally true for any float environment -- they should be referred from text} \as{also move it around to adjust the spaces properly. Can you also put the captions at the bottom}

Variance-adaptive bandit settings study exploration policies that explicitly account for the variance of observed rewards, with a rich literature focused on reward-dependent variance scaling. Audibert et al., 2009 \cite{audibert2009exploration} proposed UCB-V, which adjusts exploration using empirical variance; Gabillon et al., 2012 \cite{gabillon2012best}, and Cowan and Katehakis, 2018 \cite{cowan2018normal} further analyzed the benefits of variance-based sampling in best-arm identification. Howard et al., 2018 \cite{howard2018variance} proposed empirical Bernstein inequalities to tighten exploration bounds. Agrawal and Goyal, 2013 \cite{agrawal2013further} explored variance-aware Thompson Sampling. Although these works significantly improve learning under variable noise, they assume a single stream of data per arm and are not designed to reason about multiple parallel sources with heterogeneous noise sources and arm-independent structure, underscoring a clear gap in the literature.

\begin{table}[h]
  \caption{Summary of some related work in multi-armed bandits}
  \label{tab:related-work}
  \centering
  {\small
  \setlength{\tabcolsep}{4.5pt}
  \begin{tabular*}{\linewidth}{@{\extracolsep{\fill}}
    p{0.20\linewidth} p{0.4\linewidth} p{0.35\linewidth}
    @{}}
    \toprule
    \centering{\textit{Algorithm}} & \textit{Setting (incl. assumptions)} & \textit{Limitations} \\
    \midrule
    
\parbox[t]{\linewidth}{ \centering \emph{MF-UCB}\\[-2pt]{\footnotesize \citep{kandasamy2016multi}}} &
\begin{tabular}[t]{@{}l@{}} 
$\cdot$ Multi-fidelity $K$-armed bandit. \\[-1.5pt]
$\cdot$ Assumes known biases and fidelities. \\[-1.5pt]
\end{tabular} &
\begin{tabular}[t]{@{}l@{}} 
$\cdot$ Ignores variance entirely. \\[-1.5pt]
$\cdot$ Evaluated only on synthetic data.
\end{tabular} \\
\midrule

\parbox[t]{\linewidth}{ \centering \emph{MF-MI-Greedy}\\[-2pt]{\footnotesize \citep{song2019general}}} &
\begin{tabular}[t]{@{}l@{}} 
$\cdot$ Multi-fidelity Bayesian optimization. \\[-1.5pt]
$\cdot$ Joint GP model over fidelities. \\[-1.5pt]
$\cdot$ Cost-sensitive mutual information criterion.
\end{tabular} &
\begin{tabular}[t]{@{}l@{}} 
$\cdot$ Relies on GP prior assumptions. \\[-1.5pt]
$\cdot$ Computationally expensive. \\[-1.5pt]
\end{tabular} \\
\midrule



\parbox[t]{\linewidth}{ \centering \emph{UCB-V / PAC-UCB}\\[-2pt]{\footnotesize \citep{audibert2009exploration}}} &
\begin{tabular}[t]{@{}l@{}} 
$\cdot$ Stochastic MAB with bounded rewards. \\[-1.5pt]
$\cdot$ UCB involves an empirical variance. \\[-1.5pt]
$\cdot$ PAC-UCB variant for known horizon.
\end{tabular} &
\begin{tabular}[t]{@{}l@{}} 
$\cdot$ Single-source. \\[-1.5pt]
$\cdot$ Evaluated only on synthetic data.
\end{tabular} \\
\midrule 

\parbox[t]{\linewidth}{ \centering \emph{UGapE}\\[-2pt]{\footnotesize \citep{gabillon2012best}}} &
\begin{tabular}[t]{@{}l@{}} 
$\cdot$ Best-arm identification in stochastic MAB. \\[-1.5pt]
$\cdot$ Unified fixed-budget and fixed-confidence. \\[-1.5pt]
$\cdot$ Gap-based exploration with shared strategy.
\end{tabular} &
\begin{tabular}[t]{@{}l@{}} 
$\cdot$ Single-source. \\[-1.5pt]
$\cdot$ No empirical evaluation.
\end{tabular} \\
\midrule


\parbox[t]{\linewidth}{ \centering \emph{Thompson Sampling (TS)}\\[-2pt]{\footnotesize \citep{agrawal2013further}}} &
\begin{tabular}[t]{@{}l@{}} 
$\cdot$ Stochastic MAB with bounded rewards. \\[-1.5pt]
$\cdot$ Bayesian algorithm (uses prior), \\[-1.5pt] 
\  but frequentist regret guarantees.
\end{tabular} &
\begin{tabular}[t]{@{}l@{}} 
$\cdot$ Single-source. \\[-1.5pt]
$\cdot$ No empirical evaluation.
\end{tabular} \\
\bottomrule
\end{tabular*}}
\end{table}

\paragraph{Baselines.}
To bridge this gap, two intuitive baseline strategies naturally emerge but, unfortunately, fail dramatically in key worst-case scenarios. The first baseline selects data sources uniformly at random, effectively averaging the variance across all sources. In practice, this leads to severe performance degradation when all the sources, barring the optimal source exhibit extremely high variance—referred to as Worst-Case Instance 1 (WC-1). Here, the average variance is dominated by the high-variance sources, dramatically inflating the regret. \\
The second baseline employs a two-phase approach: it initially identifies the source with the lowest variance through a dedicated exploration phase before transitioning to standard MAB strategies. However, this method becomes problematic in instances where all sources exhibit nearly identical variances—Worst-Case Instance 2 (WC-2). In WC-2, the initial exploration phase becomes redundant and costly, resulting in significant regret without practical gain, since distinguishing among similar-variance sources is unnecessary.\\ A more detailed discussion of these problematic instances can be found in \cref{sec:app_baselines}.

\paragraph{Unresolved Questions.} These gaps in the baselines raise the following questions:
\begin{itemize}
    \item[(Q1)] Can a MAB algorithm adaptively leverage multiple data sources with unknown and differing noise variances, while achieving regret comparable to that of an oracle with access to the best source?
    \vspace{-5pt}
    \item[(Q2)] What regret guarantees are achievable in such settings, and under what conditions are they optimal?
\end{itemize}

\paragraph{Overview.}
In this paper, we address these questions by designing an algorithm that jointly explores arms and sources via confidence-based reasoning.
We propose \algmab\ (Source-Optimistic Adaptive Regret Minimization), which combines upper confidence bounds (UCBs) on arm rewards with lower confidence bounds (LCBs) on source variances.

\paragraph{Problem Statement (Informal).}
We consider a stochastic $K$-armed bandit problem with $M$ independent data sources.
Each arm $i \in [K]$ is associated with an unknown mean reward $\mu_i$, and each source $j \in [M]$ is associated with an unknown noise variance $\sigma_j^2$.\\
At each round, the learner selects an \emph{arm--source pair} $(i,j)$ and observes a noisy reward whose variance depends only on the source.\\
The learner has no prior knowledge of either the arm means or the source variances.
The objective is to minimize cumulative regret relative to the optimal arm while adaptively selecting sources to mitigate the effect of heteroscedastic noise on learning efficiency.

\paragraph{Our contributions.}
Our analysis shows that \algmab's joint LCB--UCB strategy achieves instance-dependent regret that matches, up to logarithmic factors, that of an oracle MAB with prior knowledge of the optimal source variance. Our main contributions can be summarized as follows:
\vspace{-5 pt}
\begin{enumerate}[leftmargin=13pt, parsep=0pt]
    \item \textbf{Problem formulation.} We introduce a stochastic MAB with multiple heterogeneous data sources and unknown noise variances (\cref{sec:problemSetting}).
    \item \textbf{Algorithmic preliminaries.} \as{Caps? check everywhere, we are capitalizing para headings it seems like} We develop the necessary variance and reward estimation machinery, including UCB and LCB constructions, algorithm parameters, and concentration guarantees that underpin our approach (\cref{sec:warmUp}).
    \item \textbf{Algorithm design.} We propose \algmab, a variance-aware algorithm that jointly selects arms and sources using an LCB--UCB mechanism, including a preprocessing phase that eliminates noisy sources (\cref{sec:alg_fin}).
    \item \textbf{Baseline strategies and limitations.} We analyze natural baseline strategies for source selection and identify instances of suboptimal regret (\cref{sec:app_baselines}).
    \item \textbf{Theoretical analysis.} We establish high-probability, instance-dependent regret guarantees for \algmab\ (\cref{sec:reg_fin}), showing oracle-optimality up to logarithmic factors (\cref{rem:regfin}) and improvement over the considered baselines (\cref{rem:regfinComp}). 
    \item \textbf{Empirical evaluation.} We validate our theory on synthetic benchmarks and a real-world MovieLens dataset, demonstrating improvements over natural baselines (\cref{sec:expts}).
\end{enumerate}

\section{Problem Setting}\label{sec:problemSetting}
In this section, we formalize the heterogeneous multi-source MAB problem by specifying the arms, sources, reward model, and regret formulation.

\noindent \textbf{Useful Notation.} 
Let $\R_+$ and $\N$ denote the set of positive reals and positive integers, respectively.
Let $[n] := \{1,2, \ldots n\}$, for any $n \in \N$.\\ 

\noindent Consider a stochastic bandit problem with $K$ arms $[K] \defeq \{1,\dots,K\}$ and $M$ data sources $[M] \defeq \{1,\dots,M\}$.
Each arm $i \in [K]$ has an unknown mean reward $\mu_i \in \Reals$, while each data source $j \in [M]$ has an unknown noise standard deviation $\sigma_j \in \Reals_+$ and a (possibly unknown) fourth central moment $\kappa_j \in \Reals$.

At each round $t \in [T]$, the learner selects an arm $i_t \in [K]$ and a source $j_t \in [M]$, then observes a reward
$X_t = \mu_{i_t} + \varepsilon_t$, where $\varepsilon_t \sim \cD(j_t)$ and $\cD(j_t)$ is a mean-zero distribution with variance $\sigma_{j}^2$ and fourth central moment $\kappa_j$. Equivalently, for any arm $i$ at round $t$, $\E[(X_{t}-\mu_{i_t})^4 \mid j_t]=\kappa_{j_t}$ since $X_{t}-\mu_{i_t}=\varepsilon_t$.

\vspace{-5pt}
\begin{assumption}
  For all arms $i\in[K]$, $\mu_i \in [0, \bar \mu]$ and for all sources $j \in [M]$, support$(\cD_{j_{t}}) \subseteq \left[\bar \eta, \bar \eta \right]$.
\label{ass:bounded}
\end{assumption}
\vspace{-5pt}
\noindent Let $i^* \defeq \arg\max_{i \in [K]} \mu_i$ denote the optimal arm  and $j^* \defeq \arg\min_{j \in [M]} \sigma_j$ denote the optimal source (without loss of generality, we assume \ $i^* = 1$ and $j^* = 1$). \\
We let $\mu^* \defeq \mu_{i^*}$ and $\sigma^* \defeq \sigma_{j^*}$ denote, respectively, the largest mean reward across arms and the smallest standard deviation across sources. 

\noindent The learner’s objective is to minimize the expected regret, defined as:
$
\text{Reg}_T \defeq \mathbb{E} \left[ \sum _{t=1}^T \big( \mu^* - \mu_{i_t} \big) \right]
$.

\vspace{-10pt}
\paragraph{Impact of Source Selection on Regret.}Although the expected regret does not explicitly depend on the selected sources $j_t$, source selection has a critical indirect effect.
Failure to identify and prioritize low-variance sources leads to noisier observations, slowing reward estimation, and inflating regret.
Moreover, a simple two-step strategy that first identifies the lowest-variance source and then runs a standard MAB algorithm on it is generally suboptimal, as detailed in the analysis of Baseline-2 in \cref{sec:app_baselines}.

\vspace{-10pt}
\section{Warm-Up: Parameter Estimation and Concentration}\label{sec:warmUp}
In the following subsections, we introduce estimators of mean and variance parameters, together with the confidence bounds for the different stages of \algmab \ which are further elaborated upon in \cref{sec:alg_fin}. Before we proceed, we define some more notation.
\vspace{-10pt}
\paragraph{Notation.}
For any $t \in [T]$, define

\vspace{-10pt}
\begin{equation}\label{eq:counts}
\begin{aligned}
n_i(t) \defeq \sum_{s=1}^t \ind{i_s = i}, 
\qquad 
m_j(t) \defeq \sum_{s=1}^t \ind{j_s = j},  \qquad
n_{ij}(t) \defeq \sum_{s=1}^t \ind{i_s = i}\ind{j_s = j}.
\end{aligned}
\end{equation}

where $n_i(t)$ denotes the number of times arm $i$ is selected up to time $t$,
$m_j(t)$ denotes the number of times source $j$ is selected up to time $t$,
and $n_{ij}(t)$ denotes the number of times the pair $(i,j)$ is selected up to time $t$.

\vspace{-5pt}
\paragraph{A brief description of \algmab.}\algmab \ proceeds in two stages: (i) a call to the preprocessing subroutine \algprec, which constructs a high-probability upper bound on the variances using its own variance estimator and associated confidence bounds, and (ii) the adaptive Upper Confidence Bound (UCB) - Lower Confidence Bound (LCB) selection phase, which relies on a separate estimation-and-confidence scheme for decision-making.
\vspace{-5pt}
\subsection{\algprec}
This subsection presents the estimators and confidence bounds used in the preprocessing subroutine \algprec, including the corresponding UCB/LCB definitions. A complete description of the overall algorithmic procedure is deferred to \cref{sec:alg_fin}.

\vspace{-10pt}
\paragraph{Parameter Estimation.}The empirical mean reward of arm $i$, and the empirical variance of source $j$
\emph{computed during $\algprec$}, are given by
\begin{equation}\label{eq:pp-mu-sigma}
\begin{aligned}
\hat \mu_i(t) \defeq \frac{1}{n_i(t)} \sum_{s \leq t : i_s = i} X_s, \qquad
\hat {\sigma}_{j,\mathrm{pre}}^2(t) \defeq \frac{1}{m_j(t)} \sum_{s \leq t : j_s = j}
\big( X_s - \hat \mu_{i_s}(t) \big)^2.
\end{aligned}
\end{equation}

\vspace{-5pt}
\paragraph{Confidence Bounds.}For each arm $i \in [K]$ and each source $j \in [M]$, we define the Upper Confidence Bound (UCB) and the Lower Confidence Bound (LCB) on the variance estimate ${\sigma}_{t,\mathrm{pre}}^2(j)$. Our bounds rely on the following concentration lemma.

\begin{restatable}[\algprec \ Variance Concentration]{lemma}{VarConcPreproc}
\label{lem:preprocVarConc}
Fix $\delta \in (0,1)$ and a sampling budget $\tau_p \in \mathbb{N}$.
Assume $\epsilon'' < \min\!\left\{ 6 \sigma_j^2,\; \tfrac{18\sigma_j^4}{\bar{\eta}^2} \right\}$ for each source $j \in [M]$, where $\epsilon''$ is the parameter appearing in Bernstein's inequality \as{i in Caps since its a name -- check everywhere} \cite{blm13,MATH281C_Lecture4} \as{use citep when references should come in parenthesis, it breaks the flow of reading ow, check everywhere}.
Then, with probability at least $\left(1-\delta/3\right)$,
\begin{equation*}\label{eq:var-conc}
\abs{{\hat{\sigma}_{j,\mathrm{pre}}}^2(\tau_p) - \sigma_j^2}
\;\le\;
8\bar{\eta}\,\sigma_j \sqrt{\frac{\log(4M/\delta)}{\tau_p}}.
\end{equation*}
\end{restatable}
\begin{proof}
The proof is deferred to \cref{pf:preprocVarConc}.
\end{proof}

\noindent Using \cref{lem:preprocVarConc} we can derive the following equations:
\begin{equation}\label{eq:pp_ucb_var}
\ucb_{\tau_p}^{\sigma,\mathrm{pre}}(j) \defeq \hat{\sigma}_{j, \mathrm{pre}}^2(\tau_p)
+ 8\bar{\eta}^2
\sqrt{\frac{\log(12M/\delta)}{\tau_p}},
\end{equation}
\begin{equation}\label{eq:pp_lcb_var}
\lcb_{\tau_p}^{\sigma,\mathrm{pre}}(j) \defeq \max\Biggl\{
\hat{\sigma}_{j,\mathrm{pre}}^2(\tau_p)
- 8\bar{\eta}^2
\sqrt{\frac{\log(12M/\delta)}{\tau_p}},
\,0
\Biggr\}.
\end{equation}


\vspace{-5pt}
\subsection{The Adaptive UCB-LCB Selection Phase of \algmab.} 
This subsection presents the estimators and confidence bounds used after stage (i). In short, stage (i) prunes the variance landscape by eliminating high-variance sources, ensuring that every surviving source $j$ satisfies $\sigma_j^2 \le {\sigma^*}^2 + {c^*}^2,$
where $\sigma^*$ denotes the variance of the optimal source (as defined in \cref{sec:problemSetting}) and $c^*$ is a user-chosen algorithm tolerance parameter (of the same order as $\sigma_j$). Let $S_{\mathcal{G}}$ denote the set of remaining sources, and let its cardinality be denoted by $\tilde{M} \defeq |S_{\mathcal{G}}|$. Clearly $S_{\cG} \subseteq [M].$ A complete description of the overall algorithmic procedure is deferred to \cref{sec:alg_fin}.

\begin{remark}\label{rem:cstar}
Intuitively, ${c^*}^2$ serves as a variance floor: Sources with $\sigma_j^2 < {c^*}^2$ are treated as having effective variance ${c^*}^2$ (handled as a separate case in our analysis in \cref{thm:regfintwo}), while our main regret analysis focuses on the regime $\exists j \in S_{\cG}$ such that $\sigma_j^2 \ge {c^*}^2$. This design avoids issues arising from extremely low-variance sources.
\end{remark}

\vspace{-5pt}
\paragraph{Parameter Estimators.} $\algprec$ (\cref{alg:preprocess}) queries sources using a single fixed arm, rather than multiple arms as in the adaptive UCB--LCB phase. This motivates a separate set of \textbf{unbiased pooled} \as{pooled as in?} estimators:
\begin{equation}\label{eq:mu-sigma}
\begin{aligned}
\hat \mu_i(t) &\defeq \frac{1}{n_i(t)} \sum_{s \leq t : i_s = i} X_s, 
\\
\hat\mu_{ij}(t) &\equiv \hat\mu_i(\tau_{i,j}(t))
\defeq \frac{1}{n_{ij}(t)}\sum_{\ell \in \tau_{i,j}(t)}X_\ell,
\\
\hat\sigma_{ij}^{2}(t)  &\defeq \frac{1}{n_{ij}(t) - 1} \sum_{\ell \in \tau_{i,j}(t)}
\bigl( X_\ell - \hat\mu_{ij}(t)\bigr)^2,
\\ 
\hat\sigma_{j}^2(t)
&\defeq \frac{\sum_{i \in [K]} (n_{ij}(t) - 1)\hat\sigma_{ij}^{2}(t)}{\sum_{i = 1}^{K} n_{ij}(t) - K}
= \frac{\sum_{i \in [K]} \sum_{\ell \in \tau_{i,j}(t)}
\bigl( X_\ell - \hat\mu_{ij}(t)\bigr)^2}{m_j(t) - K}.
\end{aligned}
\end{equation}
Here, $\tau_{i,j}(t) \defeq \{\ell \le t \mid j_\ell = j,\; i_\ell = i\}$ for $j \in S_{\mathcal{G}}$ and $i \in [K]$ denotes the index set of time steps up to $t$ in which arm $i$ is pulled in combination with source $j$, and its cardinality is $|\tau_{i,j}(t)| = n_{ij}(t)$.

\paragraph{A Brief Description of the Estimators.}
\begin{itemize}[leftmargin=12pt]
    \item $\hat{\mu}_{ij}(t)$ is the empirical mean reward of arm $i$ when queried specifically through source $j$.
    \item $\hat{\sigma}_{ij}^2(t)$ is the corresponding unbiased sample variance estimate for the pair $(i,j)$ computed from the same restricted observations.
    \item $\hat{\sigma}_j^2(t)$ is the pooled (unbiased) variance estimate of source $j$, obtained by aggregating the sample variances $\hat{\sigma}_{ij}^2(t)$ across all arms and normalizing by the total degrees of freedom.
\end{itemize}

\vspace{-5pt}
\paragraph{Confidence Bound on the Variance Estimate.} Recall that for each source $j \in [M]$, the fourth central moment is denoted by $\kappa_j$ \as{recall from where? If you have not defined it before, please define it in math notation}. We now define  \begin{equation}\label{eq:qdefn}
Q_j(t)=
\begin{cases}
\max \{ \kappa_j, \nu \} & \text{if } \kappa_j  \text{ is  known}\\
\max \{ \bar \eta^2 \hat \sigma_j^2(t), \nu \} & \text{if } \kappa_j  \text{ is  unknown}
\end{cases}
\end{equation}
where $\nu \in \mathbb{R}_+$ is a user-defined algorithm parameter chosen to be on the same scale as the fourth central moments $\{\kappa_j\}_{j \in [M]}$, and serves a role similar to ${c^*}^2$ (see \cref{rem:cstar}).\\
Armed with this knowledge we now define the Lower Confidence Bound (LCB) for each source $j \in S_{\cG}$ on the variance estimate $\hat {\sigma}^2_{j}(t)$. Our bound relies on the following concentration lemma.

\begin{restatable}[Source Variance Concentration]{lemma}{sigconc}\label{lem:sigConc}
Assume there exists a source $ j \in S_{\cG} \subseteq [M]$ such that $\sigma_j^2 \ge {c^*}^2$ where ${c^*} \in \mathbb{R}_+$ is a user-defined parameter of the order of $\sigma_j$. Recall $Q_j(t)$ as defined in \cref{eq:qdefn} where $\nu \in \mathbb{R}_+$ is a user-defined parameter chosen on the scale of fourth central moment $\kappa_j$.
If the source $j$, is queried for atleast $m_j(t)$ iterations where
$
m_j(t) = K + \frac{4\bar{\eta}^4 \log(3MT/\delta)}{\nu},
$  
then for any $t \in [T]$, with probability at least $1-\delta/3$,
\begin{equation}
\big| \hat{\sigma}_j^2(t) - \sigma_j^2 \big|
\;\le\;
\sqrt{\frac{Q_j(t) \log(3MT/\delta)}{m_j(t)-K}} .
\end{equation}
\end{restatable}
\begin{proof}
The proof is deferred to \cref{pf:sigConc}
\end{proof}

\noindent Using \cref{lem:sigConc} we can derive the following equation:
\begin{equation}\label{eq:lcb_var}
\lcb_t^\sigma(j) \defeq \hat \sigma_j^2(t) - 2\sqrt{\dfrac{Q_j(t) \log(3MT/\delta)}{m_j(t) - K}}
\end{equation}

\paragraph{Confidence Bound on the Mean Reward Estimate.}For each arm $i \in [K]$ and each source $j \in S_{\cG} \subseteq[M]$, we can derive an UCB on the empirical mean reward $\hat\mu_t(t)$. Before stating the UCB, we note that the bound relies on the following concentration corollary and concentration lemma.

\begin{restatable}[Variance Sandwiching]{corollary}{VarianceSandwich}
\label{cor:varsandwich}
Consider any $\delta \in(0,1)$. Assume there exists a source $ j \in S_{\cG} \subseteq [M]$ such that $\sigma_j^2 \ge {c^*}^2$ where ${c^*} \in \mathbb{R}_+$ is a user-defined parameter of the order of $\sigma_j$. If such a source $j$ is queried at least $m_j(t)$ times where $m_j(t) \geq K + \frac{16\bar\eta^4 \log(3MT/\delta)}{{c^*}^4}$ \as{use $\geq$ as you never know if the RHS is fraction, pls correct in all relevant lemma/thm statements, to make the claim statement sound and consistent}, then we have with probability $(1-\delta/3)$
\[
    \sigma_j^2 \;\le\; 2\hat\sigma_j^2(t) \;\le\; 3\sigma_j^2.
\]
\end{restatable}
\begin{proof}
The proof is deferred to \cref{pf:swcorr}
\end{proof}

\begin{restatable}[Mean Reward Concentration]{lemma}{muconc}\label{lem:muConc}
Consider any $\delta \in(0,1)$. Assume there exists a source $ j \in S_{\cG} \subseteq [M]$ such that $\sigma_j^2 \ge {c^*}^2$ and $\lcb_{\tau_p}^{\sigma,\mathrm{pre}}(j) \geq \frac{{c^*}^2}{2}$,  where ${c^*} \in \mathbb{R}_+$ is a user-defined parameter of the order of $\sigma_j$.
If an arm $i \in K$, is queried for at least $\alpha$ iterations with this source $j$ where 
$
\alpha = \frac{\bar\eta^2 \log(3KT/\delta)}{{c^*}^2}
$
then for any $t \in [T]$, we have, with probability at least $1-\delta/3$,
\begin{equation}
| \mu_i - \hat \mu_i(t) | 
         \leq \frac{2 \sqrt{ \log(3KT/\delta)  \sum_{j = 1}^M n_{ij}(t)\sigma_j^2 }}{n_i(t)} .
\end{equation}
\end{restatable}
\begin{proof}
The proof is deferred to \cref{pf:muConc}
\end{proof}

\noindent Using \cref{lem:muConc} and \cref{cor:varsandwich} we obtain the following:
\begin{equation}\label{eq:ucb_mu1}
\ucb_t^\mu(i) \defeq \hat \mu_i(t) + \frac{  2\sqrt{2\log(3KT/\delta) \sum_{j=1}^M n_{ij}(t) \hat \sigma_j^2(t)}}{n_i(t)}
\end{equation}

The UCB and LCB quantities defined above enjoy standard high-probability coverage guarantees. Although these results are direct consequences of existing concentration inequalities, we record them as \cref{cor:ucb} and~\cref{cor:varlcb} in \cref{pf:coverage} for completeness.

\section{Main Algorithm: Source-Optimistic Adaptive Regret minimization (\algmab)} 

\as{We need to merge Sec 3 and 4, Sec 3 can not be justified in isolation at this point}

\label{sec:alg_fin}
In this section, we present our proposed algorithm, \algmab, outlining its preprocessing and \emph{adaptive min-max LCB–UCB design}, and establish its regret guarantees in comparison to natural baselines.\\
 \textbf{Notations.} We define three additional quantities that will be used in the algorithm's analysis. 
For each arm $i \in [K]$, let $\Delta_i \defeq \mu^* - \mu_i$ denote the suboptimality gap. 
For each source $j \in [M]$, let $\Delta_j^\sigma \defeq \sigma_j - \sigma^*$ and $\Delta_j^{\sigma^2} \defeq \sigma_j^2 - \sigma^{*2}$ denote its excess standard deviation and excess variance relative to the most reliable source, respectively. \as{be careful with any extra whitespace}
\\
\paragraph{Some Baselines and Limitations.} 
We highlight two natural baselines for the heterogeneous multi-source bandit problem. 
The first, \emph{Baseline-1: Uniform Source MAB}, selects sources uniformly at random, leading to regret that scales with the average variance and failing in \emph{Worst-Case instance 1 (WC-1)}, where \emph{the variance of multiple, similar, high variance sources dominates performance.}\\ 
The second, \emph{Baseline-2: Two-Phase (Source-then-Arm) MAB}, first attempts to identify the lowest-variance source by fixing a single arm and running a best-arm identification procedure over the sources, before applying a standard MAB on the identified ``best source''. This strategy fails in \emph{Worst-Case instance 2 (WC-2)}, where \emph{all sources have nearly identical variances, resulting in unnecessary exploration cost.}

Detailed analyses of these baselines and their limitations are provided in \cref{sec:app_baselines}.
These insights motivate the design of our proposed algorithm, \algmab, which jointly balances arm and source exploration; further theoretical comparisons in \cref{rem:regfinComp} clarify why \algmab\ achieves strictly improved regret guarantees.



\subsection{Proposed Algorithm: \algmab} 
\label{sec:desc_algfin}

\as{what is Sec 3.2 doing? Why are we stating half of the Alg in Sec 3 and then again in Sec 4 -- seems very confusing and so much repetition.}

\vspace{-5 pt}
We now present our proposed algorithm, \algmab, \as{where is the full name?} which consists of two key components: (i) a preprocessing phase that eliminates ``bad'' or high variance sources, and (ii) an adaptive LCB–UCB procedure for joint source–arm selection.

\red{Todo: Alg desc requires some major revision and pieces should be threaded together}

\paragraph{Preprocessing to Remove ``Bad'' Sources:} 
\algprec \space is a preprocessing subroutine that eliminates high-variance sources. Each source is queried a fixed number (specifically, the runtime budget $\tau_p$) of times on the same arm, variance confidence intervals (LCB/UCB) are constructed, and a source is removed if its LCB exceeds the minimum UCB across all sources. Intuitively, sources with variance much larger than the minimum (specifically $\sigma_j^2  >  {\sigma^*}^2 + {c^*}^2$) only add noise without reducing regret. \algprec \ ensures all surviving sources satisfy $\sigma_j^2 - \sigma_*^2 < {c^*}^2$. The importance of the parameter $c^*$ is discussed in \cref{rem:cstar} and later on in \cref{rem:regfin}. This preprocessing simplifies the variance landscape, leading to tighter regret bounds and more efficient exploration. We state a key result of our analysis. 


\begin{algorithm}[h]
\caption{\algprec: Source Pruning} 
\label{alg:preprocess}
\begin{algorithmic}[1]
\STATE \textbf{Input:} Arm set: $[K]$, Feedback Sources: $[M]$, Confidence parameter: $\delta \in (0,1)$, Runtime budget: $\tau_p \in \N_+$.
\STATE \textbf{Init:} $S_{\cG} \leftarrow [M]$. Fix any arm $i_0 \in [K]$
\FOR{$j \in [M]$}
  \STATE Query source $j$ for Arm-$i_0$, $\tau_p$ times
\ENDFOR
\STATE Compute $\ucb^{\sigma,\mathrm{pre}}_{\tau_p}(j)$ and $\lcb^{\sigma,\mathrm{pre}}_{\tau_p}(j)$ for all $j \in [M]$ using \ref{eq:pp_ucb_var} and \ref{eq:pp_lcb_var}
\STATE $m \gets \min_{j \in [M]} \ucb^{\sigma, \mathrm{pre}}_{\tau_p}(j)$
\FOR{$j \in [M]$ such that $\lcb^{\sigma,\mathrm{pre}}_{\tau_p}(j) > m$} \STATE $S_{\cG} \leftarrow S_{\cG} \setminus \{j\}$ \quad \textit{// Eliminate variance source}\ENDFOR
\STATE Return $S_{\cG}$ ~~~~ \quad \textit{// Pruned set of sources}
\end{algorithmic}
\end{algorithm}

\as{remove the comments from pseudocode or format them better using texttt and right indent, etc}

\begin{restatable}[Stopping Condition of \algprec ]{theorem}{HighVarianceElimination}
\label{thm:taup_bound}
Consider any $\delta \in (0,1)$. If \algprec \ is run with runtime budget of at least $\tau_p$, where $\tau_p = \frac{1024 \bar\eta^4 \log(12M/\delta)}{{c^*}^4}$, where $c^* \in \mathbb{R}_+$ is a user defined algorithm parameter of the order of $\sigma_j$, then any source $j \in [M]$ with variance $\sigma_j^2 > {\sigma^*}^2 + {c^*}^2$,  will be eliminated with probability $\left(1-\delta/3\right)$.
\end{restatable}
\begin{proof}
The proof is provided in \cref{pf:elimCondn}.
\end{proof}

\begin{algorithm}[h]
\caption{\algmab \as{this is an abbreviation, where is the full name?}} 
\label{alg:soar}
\begin{algorithmic}[1]
\STATE \textbf{Input:} Arm set: $[K]$, Feedback Sources: $[M]$, Confidence parameter: $\delta \in (0,1)$, Exploration parameters: $ \alpha, \beta \in \N_+$, Preprocessing budget: $\tau_p \in \N_+$, Algorithm parameters: $c^*, \nu, \gamma \in \mathbb{R}_+$.
\vspace{0.5em}
\STATE Calculate $\tau_p$ from \cref{thm:taup_bound}.
\STATE $S_{\cG} \leftarrow$ \algprec$([M],[K],\delta/3,\tau_p)$  \textit{// $S_\cG$ is set of pruned sources.}
\STATE $\tilde M$ = \big|$S_{\cG}$\big| 

\vspace{0.5em}
\STATE \textbf{Initial Exploration:}
\STATE Identify a source $\bar{j} \in S_{\mathcal{G}}$ such that $\mathrm{LCB}^{\sigma,\mathrm{pre}}_{\tau_p}(\bar{j}) \ge \frac{{c^*}^2}{2}$.
\STATE For each arm $i \in [K]$, query arm $i$ using source $\bar{j}$ for $\alpha$ (as defined in \cref{eq:alphaConstraint}) rounds.
\STATE Fix any arm $i_0 \in [K]$. For each surviving source $j \in S_{\mathcal{G}}$, query arm $i_0$ using source $j$ for $\beta$ (as defined in \cref{eq:betaConstraint}) rounds.
 
\STATE At $t = M\tau_p + \tilde M\beta + K\alpha$, compute $n_i(t),\ \hat\mu_t(i),\ \ucb_t^\mu(i)$ for all $i \in [K]$, and $m_j(t), \hat\sigma_t(j),\ \lcb_t^\sigma(j)$ and $n_{ij}(t)$ for all $j \in  [S_{\cG}]$ as defined in \cref{eq:counts}, \cref{eq:lcb_var} and \cref{eq:ucb_mu1}.
\FOR{$t = M\tau_p + \tilde M\beta + K\alpha+1, \dots, T$}
  \STATE $i_t \gets \arg\max_{i \in [K]} \ucb_{t-1}^\mu(i)$, $j_t \gets \arg\min_{j \in [M]} \lcb_{t-1}^\sigma(j)$  as defined in \cref{eq:ucb_mu1} and \cref{eq:lcb_var}. 
  \STATE Pull arm $i_t$ using source $j_t$ and receive reward $X_t$ and update counts from \cref{eq:counts}. 
  \STATE \as{reduce text} Update estimators: $\hat\mu_t(i)$, $\hat\sigma_t(j)$ as defined in \cref{eq:mu-sigma}.
  \STATE Update bounds: $\ucb_t^\mu(i)$, $\lcb_t^\sigma(j)$  as defined in \cref{eq:ucb_mu1} and \cref{eq:lcb_var}.
\ENDFOR
\end{algorithmic}
\end{algorithm}


\noindent \textbf{Main Ideas of \algmab:} 
\cref{alg:soar} takes as input the number of arms $K$, number of sources $M$, a confidence parameter $\delta \in (0,1)$, exploration parameters $\alpha, \beta \in \mathbb{N}$ specifying the initial number of queries per source along with algorithm parameters $\alpha, \beta $ and $\gamma$.

\red{todo: Needs a picture showing the three major blocks in the Alg flow: Preprocessing --> Exploration ---> Max-Min LUCB}

$\gamma$ is an additional user-defined algorithm parameter chosen on the scale of $\sigma$, which appears only in the regret analysis \cref{sec:reg_fin}.

\paragraph{Initial Exploration.}
After $\algprec$ concludes, we identify a source $\bar{j}$ such that
\[
\mathrm{LCB}^{\sigma,\mathrm{pre}}_{\tau_p}(\bar{j}) \;\ge\; \frac{{c^*}^2}{2}.
\]
The existence of such a source is guaranteed under the regime assumption on $c^*$ stated in \cref{rem:cstar}.
We then query each arm $i \in [K]$ using source $\bar{j}$ at least $\alpha$ times, where
\begin{equation}\label{eq:alphaConstraint}
 \alpha \defeq \frac{\bar \eta^2 \log(3KT/\delta)}{{c^*}^2}.
\end{equation}
This ensures that the confidence bounds and estimates in \cref{lem:muConc} hold with high probability.

Similarly, we query every surviving source $j \in S_{\mathcal{G}} \subseteq [M]$ atleast $\beta$ times using any fixed arm $i_0 \in [K]$, where
\begin{equation}\label{eq:betaConstraint}
 \beta \defeq 2K
 + \frac{4 \bar{\eta}^4 \log(3MT/\delta)}{\nu}
 + \frac{16\bar \eta^4 \log(3MT/\delta)}{{c^*}^4}.
\end{equation}
This ensures that the confidence bounds and estimates in \cref{lem:sigConc} and \cref{cor:varsandwich} hold with high probability.
Overall, this initial exploration phase produces initial estimates of the arm means and source variances (together with their confidence bounds), and guarantees that each arm and surviving source is sampled sufficiently often to construct meaningful confidence intervals.\\


\noindent From round $t=M \tau_p + \tilde M\beta + K\alpha +1$ onward, the algorithm adaptively selects the arm $i_t$ with the largest upper confidence bound on mean reward and the source $j_t$ with the smallest lower confidence bound on variance, then queries the pair $(i_t,j_t)$, updates counts, and recomputes estimates and bounds. Intuitively, this UCB–LCB mechanism balances optimistic exploration of arms with cautious, variance-aware selection of sources, guiding the learner toward rewarding arms and low-noise feedback. Our analysis shows that this strategy yields regret nearly matching that of an oracle with access to the optimal low-variance source, highlighting the effectiveness of our simultaneous exploration–exploitation design.\\

\vspace{-5pt}
\section{Regret Analysis of \algmab}
\label{sec:reg_fin}
In this section, we present the main theoretical result of the paper in \cref{thm:regfin}, establishing a near-optimal regret guarantee for \algmab. As discussed in \cref{rem:regfin}, the resulting bound matches, up to logarithmic factors, the instance-dependent regret of an Oracle MAB with prior knowledge of the optimal source variance. We further compare our regret bound with that of the baseline in \cref{rem:regfinComp}, and present an additional theorem for the alternate operating regime in \cref{thm:regfintwo}.

\begin{restatable}[Main Result: Regret Analysis of \algmab]{theorem}{thmregfin}
  \label{thm:regfin}
  Consider $\delta \in (0,1)$. Assume there exists a source $ j \in S_{\cG} \subseteq [M]$ such that $\sigma_j^2 \ge {c^*}^2$ where ${c^*}, \gamma \in \mathbb{R}_+$ are  user-defined parameters of the order of $\sigma_j$. Recall $Q_j(t)$ as defined in \cref{eq:qdefn} where $\nu \in \mathbb{R}_+$ is a user-defined parameter chosen on the scale of fourth central moment $\kappa_j$.\\
  Then for any choice of preprocessing budget $\tau_p \geq \frac{1024 \cdot {\bar \eta}^4 \log(12M/\delta}{{c^*}^4}$, initial-exploration parameters $\alpha = \frac{\bar \eta^2\log(3KT/\delta)}{{c^*}^2}$, $\beta= 2K \ + \ \frac{4\bar\eta^4\log(3MT/\delta)}{\nu} + \ \frac{16\bar\eta^4\log(3MT/\delta)}{{c^*}^4}$,  the  regret of \algmab\ (\cref{alg:soar}) can be bounded by: \[
 \tilde{O}\!\Bigg(
\frac{M \bar{\eta}^4 \bar{\mu}}{{c^*}^4}
+ \frac{K \bar{\eta}^2 \bar{\mu}}{{c^*}^2}
+ KM \bar{\mu}
+ \frac{M \bar{\eta}^4 \bar{\mu}}{\nu}
+ \frac{M \bar{\eta}^4 \bar{\mu}}{{c^*}^4}
+ \sqrt{
K^2M{c^*}^2
+ \sum_{j:\,\sigma_j-\sigma^*>\gamma}
\frac{KQ_j(t)}{(\gamma+2\sigma^*)^2}
}
+ \sum_{i=2}^K\frac{(\sigma^*+\gamma)^2}{\Delta_i}
\Bigg)\]
\noindent with high probability $(1-\delta)$. \\
\algmab\ can also be shown to yield an instance-independent (worst-case) regret bound of \[\tilde O\Bigg( \frac{ M\bar \eta^4 \bar \mu}{{c^*}^4} + {\frac{K\bar \eta^2 \bar \mu}{{c^*}^2}}
+ {KM \bar \mu + \frac{M\bar \eta^4 \bar \mu}{\nu} + \frac{M \bar \eta^4 \bar \mu}{{c^*}^4}} + (\sigma^* + c^*) \sqrt{KT}\Bigg)\]  
\end{restatable}


\begin{proof}[Proof Sketch.]
We outline the main steps of the regret analysis; full details are deferred to \cref{pf:main_regret}.\\
We begin by decomposing the regret into regret incurred during the (i) preprocessing and initialization phase, and 
(ii) the adaptive LCB--UCB phase.\\

The preprocessing and initial exploration phases incur
$\tilde O\!\left(
\frac{M\bar\eta^4}{c^{*4}} +
\frac{K\bar\eta^2}{c^{*2}} +
KM +
\frac{M\bar\eta^4}{\nu}
\right)$ regret by construction.\\

\as{why so many newlines? looks weird}

For the remaining rounds, we work on a high-probability event where all
mean and variance confidence bounds hold uniformly over time.\\
Under this event, a suboptimal arm $i$ can only be selected if its UCB
exceeds the optimal arm’s mean, which yields
\[
n_i(T)
\leq
\frac{\sqrt{\log (3KT/\delta) \sum_j n_{ij}(T)\sigma_j^2}}{\Delta_i}.
\]
Similarly, a non-optimal source $j$ is selected only while its LCB overlaps
the optimal source, implying that, each source must be queried at least
\[
m_j(T)
\geq
K + \frac{Q_j(T)\log (12MT/\delta)}{(\Delta^{\sigma^2}_j)^2}.
\]
times for our algorithm to start selecting non optimal sources.\\

Substituting the bounds on $m_j(T)$ into the arm-pull inequality,
we split the variance contributions into:
(i) sources significantly worse than optimal ($\sigma_j-\sigma^*>\gamma$), and 
(ii) near-optimal sources.
The first term yields
\[
\tilde O\!\left(
\sqrt{
K^2 M c^{*2}
+
\sum_{j:\sigma_j-\sigma^*>\gamma}
\frac{K Q_j(T)}{(\gamma+2\sigma^*)^2}
}
\right),
\]
while the second term contributes the instance-dependent term
\[
\tilde O\!\left(
\sum_{i=2}^K
\frac{(\sigma^*+\gamma)^2}{\Delta_i}
\right).
\]

Adding the exploration and adaptive-phase contributions yields the stated high-probability bound\\~\\

\textbf{Worst-case bound.} \as{Caps?}
For the worst-case guarantee, note that once preprocessing eliminates
overly noisy sources, the algorithm operates only on sources with
$\sigma_j \le \sigma^*+c^*$.\\
Conditioned on the high-probability event, the problem reduces to a standard stochastic MAB with noise level $(\sigma^*+c^*)$.
Invoking a worst-case bound on the resulting self-normalized regret
gives the rate $\tilde O((\sigma^*+c^*)\sqrt{KT})$ plus lower order exploration terms. Full details are deferred to \cref{pf:main_regret}.
\end{proof}

\noindent Notably, our regret analysis shows only a negligible dependence on source variances, a consequence of the LCB–UCB selection mechanism that swiftly prioritizes lower-variance sources while maintaining aggressive reward exploration. 

\begin{remark}[Implication of \cref{thm:regfin}] \qquad
\label{rem:regfin}
\begin{enumerate}[leftmargin=13pt]
\item \textbf{Dominant Term and Optimality:}
From the instance-dependent regret bound in \cref{thm:regfin}, observe that when $\bar\eta,\bar\mu = O(1)$—a standard and reasonable assumption under bounded-moment noise models— preprocessing, exploration and estimation contribute only $O(\log T)$ (or $T$-independent) regret, while the final term $\sum_{i=2}^K \frac{(\sigma^*+\gamma)^2}{\Delta_i}$ captures the instance difficulty. Consequently, as $T$ increases, the regret is asymptotically dominated by the final term. 

Additionally since $\alpha, \beta \text{ and }\gamma$ can be chosen arbitrarily small ($~O(1)$ as per discussion below), this simplifies to $\tilde O\!\left( \sum_{i=2}^K \frac{\sigma^{*2}}{\Delta_i} \right)$, which matches the instance-dependent regret of an Oracle MAB with prior knowledge of the optimal source variance $\sigma^*$. Thus, \algmab\ achieves oracle-optimal regret up to logarithmic factors despite not knowing $\sigma^*$.

\item \textbf{Role of the parameter $\gamma$.}
The parameter $\gamma>0$ is introduced to handle degenerate regimes where $\sigma^*$ is extremely small (or zero), ensuring stability of variance-dependent confidence bounds.
In typical stochastic settings, $\sigma^*$ is bounded away from zero, and $\gamma$ can be chosen as a small constant (e.g., $\gamma \approx O(10^{-p})$ for $p\in\{2,3\}$). With this choice, $\gamma$ does not affect the dominant regret term and only plays a role when $\sigma^*\to 0$.

\item \textbf{Known fourth central moments $\kappa_j$:} \as{Caps}
When the fourth central moments $\kappa_j$ are known, the quantity $Q_j(t)$ simplifies to
$Q_j(t) = \max\{\kappa_j, \nu\}$ Substituting this into \cref{thm:regfin}, the variance-separation term becomes $\sum_{j:\,\sigma_j-\sigma^*>\gamma}
\frac{\max\{\kappa_j,\nu\}}{(\gamma+2\sigma^*)^2}$.

In this case, choosing $\nu$ smaller than the smallest $\kappa_j$ does not improve the regret bound while increasing the preprocessing cost (as $\beta \propto \nu^{-1}$). On the other hand increasing $\nu$ inflates the variance separation term. Thus, when $\kappa_j$ are known, it is natural to choose $\nu \approx \min_j \kappa_j$, yielding the tightest bound.

\item \textbf{Gaussian Noise \as{why not SubGaussian directly?}:} In particular, for Gaussian rewards where $\kappa_j = 3\sigma_j^4$, the variance-separation contributes $\sqrt{
K^2M{c^*}^2
+
K\sum\limits_{j:\,\sigma_j-\sigma^*>\gamma}\frac{3\sigma_j^4}{(\gamma+2\sigma^*)^2}
}$ to the regret.
Note that this contribution is independent of $T$ (for fixed problem parameters), and is therefore asymptotically dominated by the leading instance-dependent term $\sum_{i=2}^K \frac{(\sigma^*+\gamma)^2}{\Delta_i}$. \as{add a line then what is the orderwise final regret in this case}

\item \textbf{Unknown fourth moments $\kappa_j$:} \as{Caps are all messed up!}
When $\kappa_j$ are unknown, $Q_j(t)$ is defined as $Q_j(t) = \max\{\bar\eta^2 \hat\sigma_j^2(t), \nu\}$
where $\nu>0$ acts as a floor ensuring stability of variance-dependent confidence bounds.

In this setting, the choice of $\nu$ induces an explicit tradeoff. On one hand, a larger $\nu$ increases the numerator of the regret bound through the term $\sum_{j:\,\sigma_j-\sigma^*>\gamma}
\frac{Q_j(t)}{(\gamma+2\sigma^*)^2},$ \as{watch $t$}
potentially inflating the regret when empirical variance estimates are small. On the other hand, $\nu$ appears in the denominator of the exploration parameter $\beta$ so decreasing $\nu$ increases the required exploration budget and hence the corresponding exploration budget.

Notably, this tradeoff mirrors that of the parameter $c^*$, which controls the balance between reliable variance separation and exploration cost; in practice, both $\nu$ and $c^*$ can be chosen as small constants $(\sim O(1))$ without affecting the leading-order regret.
\end{enumerate}
Taken together, these observations show that despite unknown variance and higher-order moments, \algmab\ attains oracle-optimal instance-dependent regret up to logarithmic factors.

\end{remark}

\vspace{-1pt}
\begin{remark}[Improved Regret Bound and Comparison to Baselines]
  \label{rem:regfinComp}
  
   Comparing the regret guarantee in \cref{thm:regfin} with the baselines in \cref{sec:app_baselines}, 
  we find that \algmab\ improves over both Baseline-1 (Uniform Source MAB) and Baseline-2 (Two-Phase MAB) under the worst-case instances \textbf{WC-1} and \textbf{WC-2}. \\
\as{new line + arxiv format}
  
\noindent Consequently, \algmab\ incurs an instance-dependent regret of 
\[
 \tilde{O}\!\Bigg( \frac{K \bar{\eta}^2 \bar{\mu}}{{c^*}^2}
+ KM \bar{\mu}
+ \frac{M \bar{\eta}^4 \bar{\mu}}{\nu}
+ \frac{M \bar{\eta}^4 \bar{\mu}}{{c^*}^4}
+ \sqrt{K}\sqrt{
KM{c^*}^2
+ \sum_{j:\,\sigma_j-\sigma^*>\gamma}
\frac{Q_j(t)}{(\gamma+2\sigma^*)^2}
}
+ \sum_{i=2}^K\frac{(\sigma^*+\gamma)^2}{\Delta_i}
\Bigg).
\] 
In contrast, Baseline-1 (Uniform Source MAB) suffers, under \textbf{WC-1}, an instance-dependent regret of
\[
    \tilde{O}\!\left(
        \sigma_{\max}^2 \sum \limits_{i \neq i^*} \frac{\log(MKT)}{\Delta_i^{}}
    \right),
\]
and Baseline-2 (Two-Phase MAB) incurs
\[
    \tilde{O}\!\left(
        \sum \limits_{j\neq j^*} \frac{\bar \mu}{(\Delta_j^{\sigma^2})^2}
        + {\sigma^*}^2 \sum \limits_{i \neq i^*} \frac{1}{\Delta_i}
    \right)
\]
whose first term can blow up under \textbf{WC-2}, that is, when the variance gaps $\{ \Delta_j^{\sigma^2}\}$ are very small.
\\

\noindent Moreover, \algmab\ attains a worst-case (instance-independent) regret of 
\[\tilde O\Bigg( {\frac{K\bar \eta^2 \bar \mu}{{c^*}^2}}
+ {KM \bar \mu + \frac{M\bar \eta^4 \bar \mu}{\nu} + \frac{M \bar \eta^4 \bar \mu}{{c^*}^4}} + (\sigma^* + c^*) \sqrt{KT}\Bigg)\]
compared to Baseline-1 (Uniform Source MAB), which under \textbf{WC-1} has worst-case regret
\[
    \tilde{O}\!\left(
        \sigma_{\max} \sqrt{KT}
    \right),
\]
and Baseline-2 (Two-Phase MAB), which has a worst-case regret of 
\[
    \tilde{O}\!\left(
        \sum \limits _{j\neq j^*} \frac{\bar \mu}{(\Delta_j^{\sigma^2})^2}
        + \sigma^* \sqrt{KT}
    \right).
\]\\

  Unlike Baseline-1 (Uniform Source MAB), whose regret scales with the average variance and can be dominated by multiple, similar noisy sources, \algmab\ adaptively limits the effect of high-variance sources. 
  
  Similarly, Baseline-2 (Two-Phase MAB) incurs unnecessary regret in regimes where source variances are nearly identical, as it dedicates a full phase to distinguishing them. By balancing arm and source exploration, \algmab\ avoids this inefficiency.
 A detailed analysis of these baselines is found in \cref{sec:app_baselines}. Corresponding plots verifying these claims are presented in \cref{sec:expts}.
\end{remark}

Next. we analyze the complementary regime in which all surviving sources have variance below the threshold ${c^*}^2$, and establish the corresponding regret guarantees for \algmab.

\begin{restatable}[Regret Analysis of \algmab \ in `Low-Noise Regime']{theorem}{thmregfintwo}
  \label{thm:regfintwo}
  Consider $\delta \in (0,1)$. Assume that all sources $ j \in S_{\cG} \subseteq [M]$ are such that $\sigma_j^2 < {c^*}^2$ where ${c^*} \in \mathbb{R}_+$ is a user-defined parameter of the order of $\sigma_j$. 
  Then for any choice of preprocessing budget $\tau_p \geq \frac{1024 \cdot {\bar \eta}^4 \log(12M/\delta}{{c^*}^4}$, inital exploration parameter $\alpha = \bar \eta^2 {c^*}^2 \log(3KT/\delta)$, then  the  regret of \algmab\ (\cref{alg:soar}) can be bounded by \[\tilde O\left( \frac{ M\bar \eta^4 \bar \mu}{{c^*}^4} + K\bar \eta^2 {c^*}^2 \bar \mu +  \sum_{i = 2}^K \frac{ {c^*}^{2}}{\Delta_i}  \right)\] with high probability $(1-\delta)$.\\ 
In this regime, \algmab\ can also be shown to yield an instance-independent (worst-case) regret bound of \[\tilde O\biggn{ \frac{ M\bar \eta^4 \bar \mu}{{c^*}^4} + {K \bar \eta^2 {c^*}^2 \bar \mu} + K{c^*}^2\sqrt{T} }\] 
\vspace{-1pt}
\end{restatable}
\begin{proof} 
\vspace{-10pt}
The proof is deferred to \cref{pf:regfin2}.
\end{proof}

\section{Experiments}
\label{sec:expts}

In this section, we first examine the behavior of our approach as the number of arms and sources varies, and then evaluate its performance against the previously introduced baselines. Arm rewards are modeled as Gaussian random variables with fixed means and source-dependent variances, and shaded regions in the plots denote $95\%$ confidence intervals.  In all setups, we observe an initial linear growth in regret, which arises from the preprocessing phase combined with the initial exploration rounds.
A more detailed description of the experimental setup is provided in \cref{sec:app_expts}.

\begin{figure}[H]
    \centering
    \begin{minipage}[t]{0.48\linewidth}
        \centering
        \includegraphics[width=\linewidth, height=4cm]{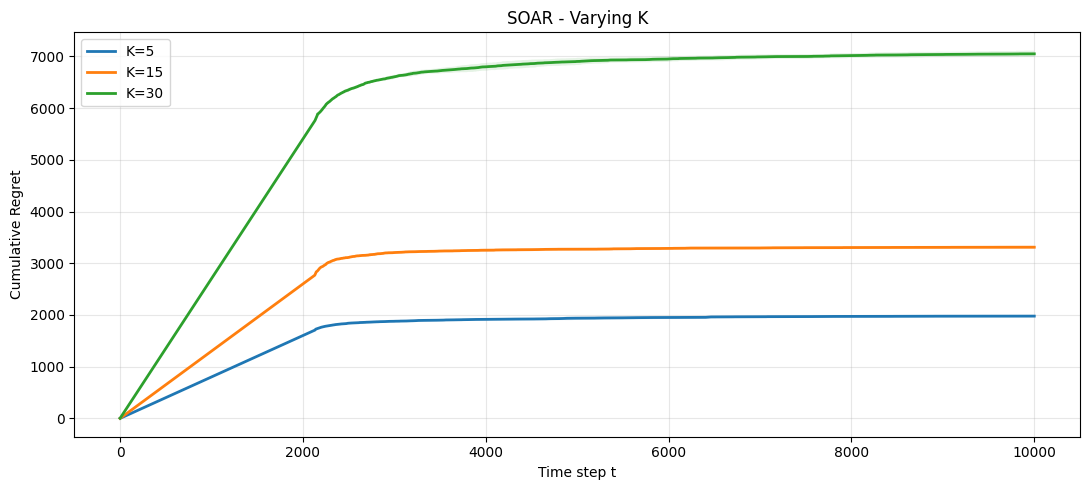}
        
        \caption{Regret of \algmab\ with varying number of arms $K \in \{5,15,30\}$}
        \label{fig:varyingK}
    \end{minipage}
    \hfill
    \begin{minipage}[t]{0.48\linewidth}
        \centering
        \includegraphics[width=\linewidth, height=4cm]{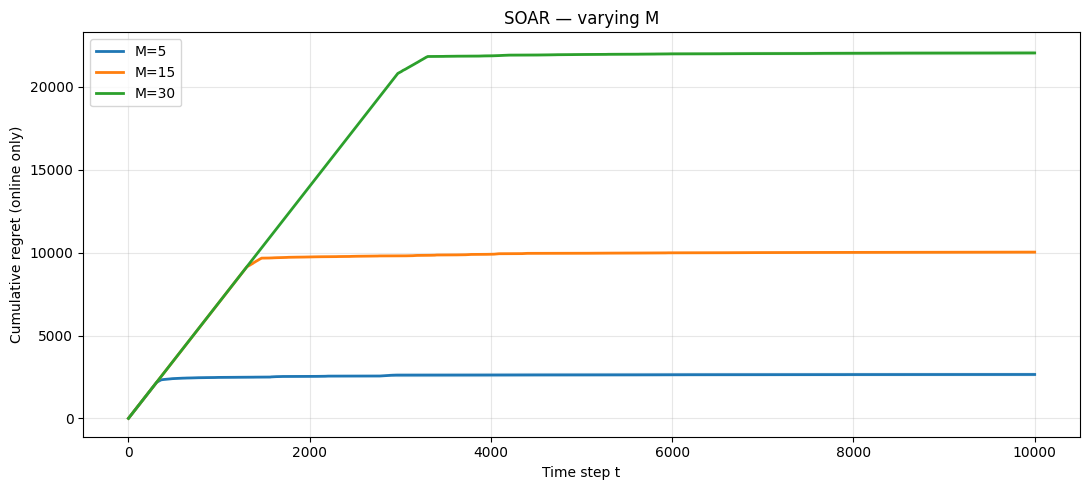}
        
        \caption{Regret of \algmab\ with varying number of sources $M \in \{5,15,30\}$}
        \label{fig:varyingM}
    \end{minipage}
\end{figure}
\begin{figure}[H]
    \centering
    \begin{minipage}[t]{0.48\linewidth}
        \centering
        \includegraphics[width=\linewidth, height = 4cm]{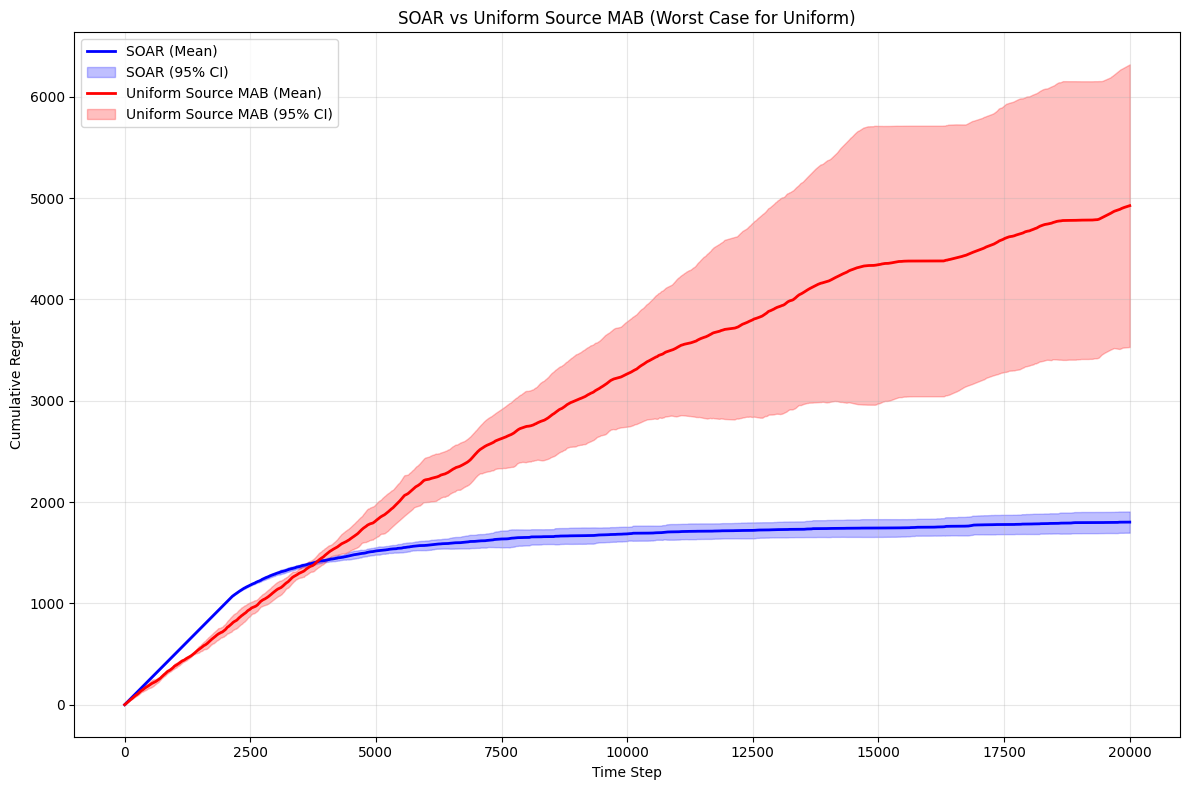}
        
        \caption{\algmab \space vs Baseline-1: WC-1 \as{say on WC1, `:' is confusing}}
        \label{fig:varyingK}
    \end{minipage}
    \hfill
    \begin{minipage}[t]{0.48\linewidth}
        \centering
        \includegraphics[width=\linewidth, height = 4cm]{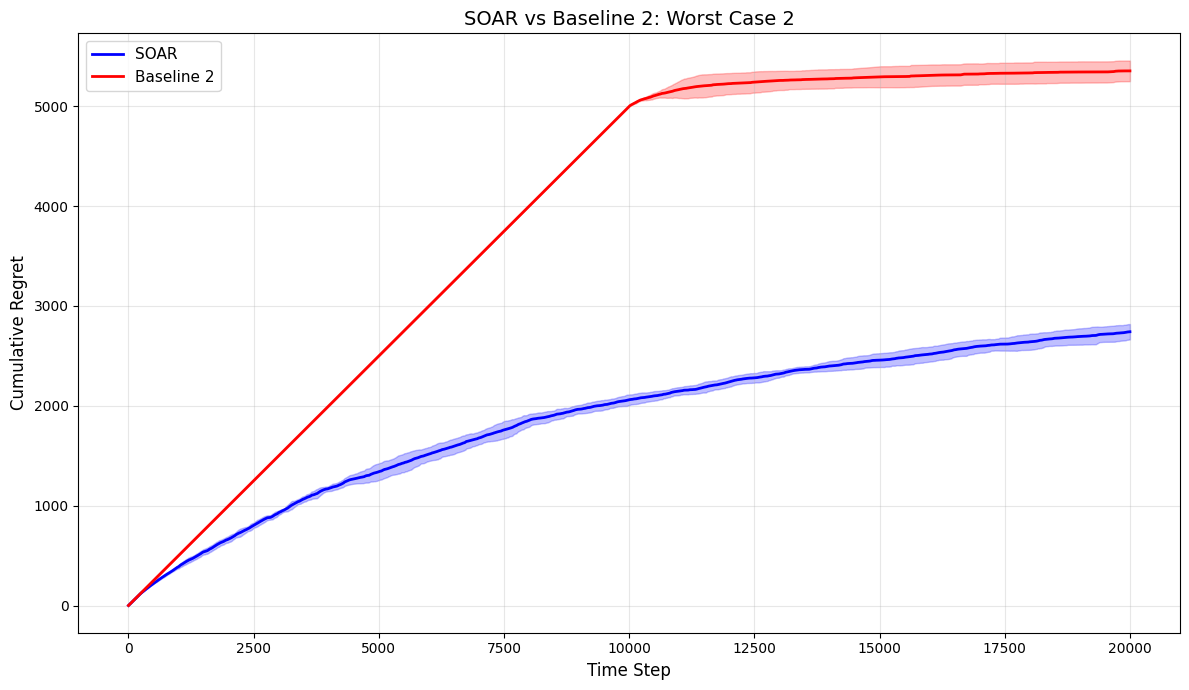}
        
        \caption{\algmab\space vs Baseline-2: WC-2}
        \label{fig:varyingM}
    \end{minipage}
\end{figure}

\as{Keep this in mind for next paper: Pls adjust the axis, legends, title in the plots to make them readable -- very hard to read now. Its ok for this paper for now}
  
\paragraph{MovieLens Panel} To illustrate source adaptivity on real ratings, we build a small fixed panel from the MovieLens\;25M dataset \cite{harper2015movielens} \as{citep}: we select 15 reviewers who rated the same 500 movies (7{,}500 ratings in total). We model each movie $i$ as an arm with an unknown mean reward $\mu_i \in [0,5]$, corresponding to its mean rating in the panel.
For each reviewer, we compute the review variance of the reviewer and take it as the source variance in our algorithm.

\as{floating white spaces look unprofessional}

\begin{figure}[H]
    \centering
    \begin{minipage}[t]{0.48\linewidth}
        \centering
        \includegraphics[width=\linewidth, height=4cm]{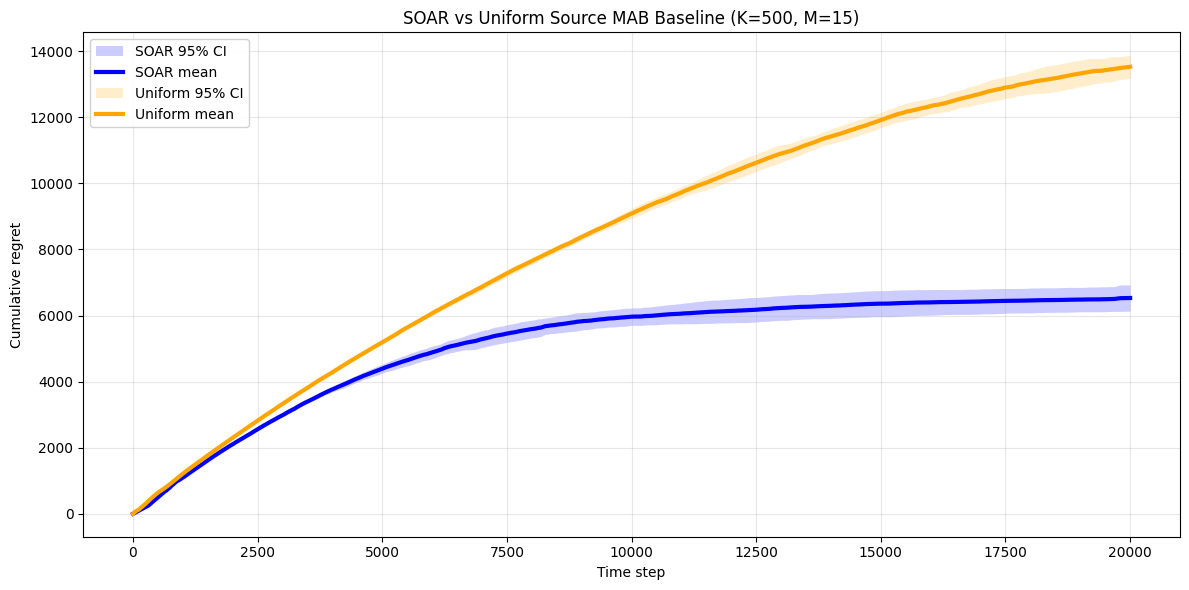}
        \caption{\algmab\ vs.\ Baseline-1 on MovieLens.}
        \label{fig:ml_b1}
    \end{minipage}\hfill
    \begin{minipage}[t]{0.48\linewidth}
        \centering
        \includegraphics[width=\linewidth, height=4cm]{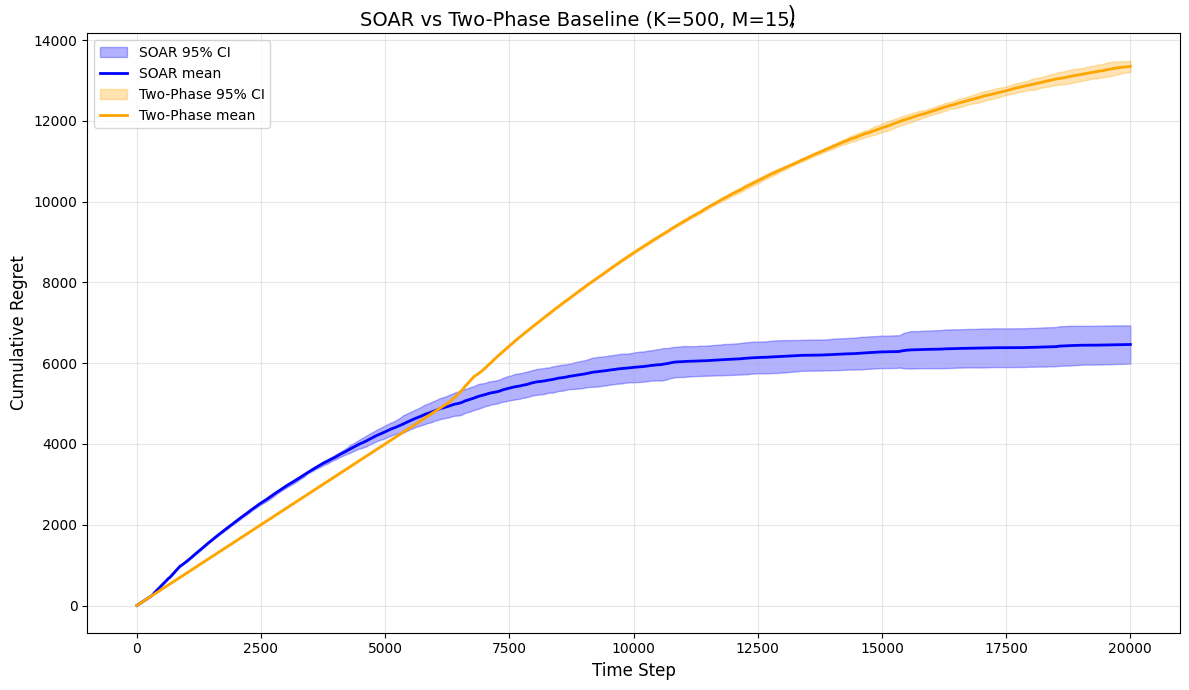}
        \caption{\algmab\ vs.\ Baseline-2 on MovieLens.}
        \label{fig:ml_b2}
    \end{minipage}
\end{figure}

\noindent
\textbf{Findings.} On the MovieLens panel, \algmab\ rapidly concentrates on the lowest-variance reviewer while converging to the highest-rated movie, achieving substantially lower cumulative regret than both baselines (Figures~\ref{fig:ml_b1}–\ref{fig:ml_b2}). This aligns with our theory, that  \algmab\ preferentially routes pulls to low-variance sources.

\section{Discussion}
\label{sec:concl}
We studied the problem of online learning with multiple data sources under heteroscedastic noise, aiming to minimize regret through adaptive source and arm selection.
Our proposed algorithm, \algmab, simultaneously explores and exploits both the data sources (with varying unknown variances) and the reward-generating arms, significantly improving over conventional baselines. 

\vspace{-10pt}
\paragraph{Future Scope.}
A compelling direction for future research is to extend our adaptive learning framework to more general contextual bandit or reinforcement learning scenarios. Additionally, developing algorithms that gracefully handle non-stationary environments, where the quality of data sources changes over time, remains an open challenge that would substantially broaden the applicability of our adaptive learning framework. Extending the framework to other feedback models, e.g. demonstration or relative feedback, might also be interesting future work.

\section*{Acknowledgement}
Thanks to Thomas Kleine Buening for some of the initial discussions towards developing the LCB-UCB idea of SOAR.


\bibliographystyle{plainnat}
\bibliography{icml_refs.bib}

\newpage 

\appendix


\section*{\centering\LARGE{Supplementary: \papertitle}}
\vspace*{1cm}

\section{Few Baselines and their Limitations}
\label{sec:app_baselines}
\as{Make it the new Sec 3. I already commented above, present Sec 3 and 4 needs to be stitched together}

\as{Start with a sentence saying, this section aims to discuss the theoretical guarantees of SOAR with existing baselines, however there have not been any prior work that addressed our problem setting, as already discussed in the Rel-work (cref the Sec). We hence try to analyze the performance of SOAR with two UCB-based baselines, which we designed as follows:}

In this section, we recall our earlier defined baseline algorithms and highlight their limitations, motivating the design of our simultaneous source-arm exploration strategy.

\subsection{Baseline-1: Uniform-UCB (UUCB)}

We start with a brief description of the algorithm and analyze its regret bound, especially for different types of problem instances.

\as{this is optional but good to have: I think calling the instances WC is overselling. We should just call them Instance-1, Instance-2, or even better $\cI_{\texttt{UUCB}}$, etc. Since these instances are not really hard, just that the baselines are really poor, we should not give the impression that we are testing the baselines on `very hard instances' adversarially, but rather convey how they easily fail even on basic, naturally occurring problem instances. Use these wordings if appropriate in the remarks}

\vspace{-10pt}
\paragraph{Algorithm Description: Uniform-UCB (UUCB).}
This baseline selects a data source uniformly at random at each round $t$ and uses the corresponding observations to update reward estimates and apply standard multi-armed bandit (MAB) routines. Under this uniform selection strategy, each arm $i \in [K]$ experiences ``\emph{an effective variance}'' equal to the average variance across all sources:
$
    \tilde \sigma^2 \defeq \frac{1}{M}\sum_{j = 1}^M \sigma_j^2.
$

\as{please a small pseudocode, it will look nice it you can wrap them in a minipage environment: LHS: Alg descriptopn and RHS: Pseudocode, but never mind if that takes time}

\paragraph{Regret Analysis of UUCB:} Following standard upper confidence bound (UCB)-based analyses for MAB algorithms~\cite{CsabaNotes18} \as{this should be citep -- understand the difference of cite and citep and use appropriately}, the regret for this uniform-source MAB algorithm can be bounded as
$
    O\left(\tilde \sigma \sqrt{KT}\right),
$
or, more specifically, using instance-dependent regret bounds as
$
    O\left(\tilde \sigma^2 \sum \limits_{i \neq i^*} \frac{\log(MKT)}{\Delta_{i}^{ }}\right).
$
However, this uniform averaging approach can be gravely suboptimal in certain instances. \\

Consider a worst-case scenario (call this \emph{worst-case instance 1 (WC-1)}) where all sources except the optimal source have a similar, significantly high variance say $\sigma_{\text{max}}^2$. 
In this case, the effective variance $\tilde \sigma^2$ is dominated by the multiple high variance sources:
$
    \tilde \sigma^2 = \frac{({M-1})\sigma_{\text{max}}^2}M \approx \sigma_\text{max}^2.
$\\
Consequently, the instance dependent regret bound deteriorates to $O\left(\sigma_{\text{max}}^2 \sum \limits_{i\neq i^*} \frac{\log(MKT)}{\Delta^{}_i} \right)$ while the instance independent bound becomes $O(\sigma_{\text{max}} \sqrt{KT})$. 
Thus, the learner incurs undesirably high regret, as the estimation accuracy is effectively governed by the extremely noisy non-optimal sources.

\subsection{Baseline-2: Explore-then-Commit UCB (ETC-UCB)}

\paragraph{Algorithm Description: Two-Phase (Source-then-Arm) MAB.}
In this baseline, the learner first attempts to identify the data source with the smallest variance by fixing a single arm and running a best-arm identification (BAI) algorithm across the sources. After this initial phase, the learner commits exclusively to the identified ``good'' variance source and runs a standard MAB algorithm on the $K$ arms, querying arm rewards solely from this selected source. 

\as{Add a similar pseudocode}

\paragraph{Regret Analysis of ETC-UCB:} 
To analyze the regret incurred by this two-phase strategy, we note that the first phase (source identification) incurs a regret of approximately
$
    O\left(\sum\limits_{j \neq j^*} \frac{\bar \mu\log(MKT)}{\max\{\epsilon^2,(\Delta_j^{\sigma^2 })^2\}}\right),
$
where $\epsilon \in \mathbb{R}_+$ is a user-specified error tolerance parameter determining the desired precision (PAC-optimality) in identifying the minimal-variance source. Importantly, during each round of the first phase, the algorithm potentially incurs $O(1)$ \as{this is not true as we removed bdd by 1 assumption} regret (which in turn is bounded by $\bar \mu$) since it continues to query a fixed arm and does not make progress in identifying the optimal reward arm $i^* \in [K]$.\\

In the second phase (reward-arm identification and exploitation), the algorithm's regret can be bounded as
$
    O\left((\sigma^*+\epsilon)^2 \sum\limits_{i \neq i^*} \frac{\log(MKT)}{(\Delta_{i})}\right),
$
leading to an overall regret (summing two phases) of:
\[
    O\left(\sum_{j \neq j^*} \frac{\bar \mu \log(MKT)}{\max\{\epsilon^2,(\Delta_j^{\sigma^2})^2\}} + (\sigma^*+\epsilon)^2 \sum_{i \neq i^*} \frac{\log(MKT)}{\Delta_{i}}\right).
\]

We remark that $\epsilon$ is a tunable parameter whose choice significantly impacts the performance of this approach. However, it is hard for the learner to optimize the value of $\epsilon$ without the knowledge of the variance gaps $\{\Delta^{\sigma^2}_j\}_{j \in [M]}$. \\

Consider now the worst-case scenario for this two-phase algorithm (call this \emph{worst-case instance 2 (WC-2)}). This occurs when all data sources have roughly identical variances 
\[
    \sigma_1^2 \approx \sigma_2^2 \approx \dots \approx \sigma_M^2.
\]
In this regime all variance gaps satisfy \(\Delta_j^{\sigma^2} \approx 0\), so setting \(\epsilon \to 0\) could be particularly
disastrous --- the factor \(1/\max\{\epsilon^2,(\Delta_j^{\sigma^2})^2\}\) in our regret bound effectively blows up, and the
algorithm is forced to keep sampling all sources in an attempt to detect non-existent variance differences --- finally leading to an
arbitrarily poor bound.

In such scenarios, the learner could have selected any data source (or even randomly selected sources at each round, as in the uniform source baseline described above), since there is no practical benefit in spending effort to distinguish among equally good sources. Yet, the two-phase algorithm will unnecessarily incur substantial regret in the initial source-identification phase while trying to pinpoint the optimal (lowest-variance) source, accumulating a high regret of roughly
\[
    O\left(\sum_{j \neq j^*} \frac{\bar \mu \log(MKT)}{\max\{\epsilon^2,(\Delta_j^{\sigma^2})^2\}}\right).
\]

Here, note that $\Delta_j^{\sigma^2} \approx 0$ for all sources $j$, which exacerbates the incurred regret. Moreover, since the learner does not initially know the minimal source variance $\sigma^*$, choosing an appropriate $\epsilon$ to balance this trade-off and minimize regret becomes challenging, leading to additional practical issues.

This dependence on $\epsilon$ becomes clearer when we look at two extreme choices. Suppose the learner naively fixes $\epsilon = c_1 > 0$ independently of the instance, and consider a WC-2 instance in which all source variances are nearly identical, so that $\Delta_j^{\sigma^2} \approx 0$ for all
$j \neq j^*$. Then, for each such source, $\max\{\epsilon^2, (\Delta_j^{\sigma^2})^2\}\approx \epsilon^2 = c_1^2,$
and the corresponding contribution to the regret satisfies
\[
  \sum_{j\neq j^*} \frac{\bar\mu \log(MKT)}{\max\{\epsilon^2,(\Delta_j^{\sigma^2})^2\}}
  \approx \,\sum_{j\neq j^*}\frac{ \bar\mu\log(MKT)}{c_1^2},
\]
which can be made arbitrarily large if $c_1$ is very small. Conversely, if the learner picks $\epsilon$ overly large, then the term
\[
  (\sigma^*+\epsilon)^2 \sum_{i\neq i^*} \frac{\log(MKT)}{\Delta_i},
\]
is dominated by $\epsilon^2$, and the regret again scales poorly. Thus any fixed, instance-independent choice $\epsilon = c_1$ can be far from optimal across different problem instances.

These analyses reveal crucial limitations in both approaches, highlighting that an optimal strategy must \emph{simultaneously} balance exploration and exploitation across data sources and reward arms. A well-designed algorithm must dynamically and adaptively explore low-variance sources, while simultaneously identifying the best-performing arms as quickly as possible to minimize overall regret. 

\newpage
\section{Proofs Related to \algprec}
\subsection{Proof of \algprec \ Variance Concentration Lemma}\label{pf:preprocVarConc}
\VarConcPreproc*
\begin{proof}
\as{Pls write some introductory lines before starting the math abruptly}

\begin{align*}
 \abs{\hat \sigma_{j, \mathrm{pre}}^2(\tau_p) - \sigma_j^2} & = \abs{\frac{1}{\tau_p} \sum_{k=1}^{\tau_p} \big( X_k - \hat \mu_{i}(\tau_p) \big)^2 - \sigma_j^2},
 \\
 & = \abs{ \frac{1}{\tau_p} \sum_{k=1}^{\tau_p} \big( X_k - \mu_i + \mu_i -  \hat \mu_{i}(\tau_p) \big)^2 - \sigma_j^2},
 \\
 & \leq \abs{{\frac{2}{\tau_p} \sum_{k=1}^{\tau_p}  \big( X_k - \mu_i)^2}  + {\frac{2}{\tau_p} \sum_{k=1}^{\tau_p} (\mu_i -  \hat \mu_{i}(\tau_p) \big)^2} - \sigma_j^2},
 \\
  & \leq \underbrace{ \abs{ \frac{2}{\tau_p} \sum_{k=1}^{\tau_p}  \big( X_k - \mu_{i})^2 - \sigma_j^2 }}_{\text{(I)} \leq \epsilon/2}  + \abs{\underbrace{\frac{2}{\tau_p} \sum_{k=1}^{\tau_p} (\mu_i -  \hat \mu_{i}(\tau_p) \big)^2}_{\text{(II)} \leq \epsilon/2}}.
\end{align*}

\paragraph{Bounding Term (I):} To bound Term I we will use Bernstein's Inequality, stated below.
\begin{restatable}[Bernstein's Inequality]{theorem}{BernsteinIneq}
\label{thm:bernstein}
\citep[Cor.~2.11]{blm13}; \citep[Theorem 1.2]{MATH281C_Lecture4}
Suppose $Z_1, \ldots, Z_n$ are independent random variables with finite variances, and suppose that 
\[
\max_{1 \leq k \leq n} |Z_k| \leq b
\]
almost surely for some constant $b > 0$. Let 
\[
V = \sum_{k=1}^n \mathbb{E}[Z_k^2].
\]
Then, for every $t \geq 0$,
\[
\Pr\!\left( \sum_{k=1}^{n} (Z_k - \mathbb{E}[Z_k]) \geq t \right) 
\;\leq\; \exp\!\left( -\frac{t^2}{2\big(V + \tfrac{1}{3}bt\big)} \right),
\]
and
\[
\Pr\!\left( \sum_{k=1}^{n} (Z_k - \mathbb{E}[Z_k]) \leq -t \right) 
\;\leq\; \exp\!\left( -\frac{t^2}{2\big(V + \tfrac{1}{3}bt\big)} \right).
\]
\end{restatable}

Now, coming back to our proof:\\
\noindent Fix a source $j \in [M]$ and arm $i \in [K]$. Consider $Z_k = (X_k - \mu_i)^2$, where  $X_k = \mu_{i} + \varepsilon$ and $\varepsilon \sim \cD(j)$ is drawn from an (unknown) underlying noise distribution $\cD(j)$ of the selected data-source $j \in [M]$.\\
From our problem setting and \cref{ass:bounded}, it is evident $\abs{Z_k} \leq {\bar \eta}^2$. Now \\
\[
V = \sum_{k=1}^{\tau_p} \mathbb{E}[Z_k^2] \leq \sum_{k=1}^{\tau_p}  {\bar \eta}^2 E[Z_k] \leq \tau_p \ {\bar \eta}^2\sigma_j^2. 
\]
Applying Bernstein's, we have with probability  $\left(1-\dfrac{\delta}{6}\right)$:

\begin{align*}
\Pr\!\left( \abs{\sum_{k=1}^{\tau_p} \biggn{(X_k - \mu_i)^2 - \sigma_j^2}} \geq \epsilon \right) 
\;&\leq\; 2\exp\!\left( -\frac{\epsilon^2}{2\big(\tau_p \ {\bar \eta}^2\sigma_j^2 + \tfrac{1}{3}{\bar \eta}^2\epsilon \big)} \right), 
\\
\equiv \Pr\!\left( \abs{\frac{1}{\tau_p}\sum_{k=1}^{\tau_p} {(X_k - \mu_i)^2} - \sigma_j^2} \geq \frac{\epsilon}{\tau_p } \right) 
\;&\leq\; 2\exp\!\left( -\frac{\epsilon^2}{2\big(\tau_p \ {\bar \eta}^2 \sigma_j^2 + \tfrac{1}{3}{\bar \eta}^2\epsilon \big)} \right),
\\
\implies \Pr\!\left( \abs{\frac{1}{\tau_p}\sum_{k=1}^{\tau_p} {(X_k - \mu_i)^2} - \sigma_j^2} \geq {\epsilon'} \right) 
\;&\leq\; 2\exp\!\left( -\frac{{\tau_p}^2{\epsilon'}^2}{2\big({\tau_p} \ {\bar \eta}^2\sigma_j^2 + \tfrac{1}{3}{\bar \eta}^2 \ {\tau_p} \ \epsilon' \big)} \right), ~~\left[\epsilon' = \frac{\epsilon}{\tau_p}\right]
\end{align*}

Replace $\epsilon'$ by $\epsilon''/2$, \as{we get:}
\[
\Pr\!\left( \abs{\frac{1}{\tau_p}\sum_{k=1}^{\tau_p} {(X_k - \mu_i)^2} - \sigma_j^2} \geq {\dfrac{\epsilon''}{2}} \right) 
\;\leq\; 2\exp\!\left( -\frac{{\tau_p} \ {\epsilon''}^2}{8\big( {\bar \eta}^2\sigma_j^2 + \tfrac{1}{6}{\bar \eta}^2\epsilon \big)} \right).
\]

Choose $\epsilon''$ such that $ \dfrac{\sigma_j^2}{\epsilon''^2} > \dfrac{1}{6\epsilon} $ i.e $\epsilon'' < 6 \sigma_j^2$, This is equivalent to $\dfrac{\delta}{6} \geq 2 \exp\!\left( -\dfrac{\tau_p \ \epsilon''^2}{16{\bar \eta}^2 \sigma_j^2} \right)$
$\implies \epsilon'' \le 4{\bar \eta}\sigma_j\sqrt{\dfrac{\log(12/\delta)}{\tau_p}}.$

\noindent Thus with $\epsilon'' < 6 \sigma_j^2$, and a union bound over all sources $j \in [M]$ we have with probability $\left(1-\dfrac{\delta}{6}\right)$,
\[
 \abs{\frac{2}{\tau_p}\sum_{k=1}^{\tau_p} {(X_k - \mu_i)^2} - \sigma_j^2} \leq 4{\bar \eta}\sigma_j\sqrt{\dfrac{\log(12M/\delta)}{\tau_p}} .
\]

\paragraph{Bounding Term (II):}
Consider $D_k = (X_k - \mu_i)^2$, where  $X_k = \mu_{i} + \varepsilon$ and $\varepsilon \sim \cD(j)$ is drawn from an (unknown) underlying noise distribution $\cD(j)$ of the selected data-source $j \in [M]$. Similar to how we bounded term(I), it is evident $\abs{D_k} \leq {\bar \eta}$. Additionally conditional variance $V = \sum^{\tau_p}_{k=1} E[D_k^2] = \tau_p\cdot\sigma_j^2$\\
Applying Bernstein's, we have with probability $\left(\delta/6\right)$:

\begin{align*}
\Pr\!\left( \dfrac{1}{\tau_p}\abs{\sum_{k=1}^{\tau_p} \biggn{X_k - \mu_i}} \geq \epsilon'/2 \right) 
\;&\leq\; 2\exp\!\left( -\frac{{{\tau_p} \ \epsilon'}^2}{8\big( \sigma_j^2 + \tfrac{1}{6}{\bar \eta}\epsilon' \big)} \right), ~~\left[ \epsilon' = \frac{\epsilon}{\tau_p}\right]
\\
\implies 
\Pr\!\left( \abs{ \hat \mu_{i}(\tau_p) - \mu_i} \geq \epsilon'/2 \right) 
\;&\leq\; 2\exp\!\left( -\frac{{{\tau_p} \ \epsilon'}^2}{8\big( \sigma_j^2 + \tfrac{1}{6}{\bar \eta}\epsilon' \big)} \right),
\\
\implies 
\Pr\!\left( \abs{ \hat \mu_{i}(\tau_p) - \mu_i}^2 \geq {\epsilon'}^2/4 \right) 
\;&\leq\; 2\exp\!\left( -\frac{{{\tau_p} \ \epsilon'}^2}{8\big( \sigma_j^2 + \tfrac{1}{6}{\bar \eta}\epsilon' \big)} \right),
\end{align*}    

Choose $\epsilon'' = {\epsilon'}^2 /2 \implies \text{\as{replace or equivalently in words}} \epsilon' = \sqrt{2\epsilon''}$, \as{we get (this is in many places, check)}
\[
\Pr\!\left( \abs{ \hat \mu_{i}(\tau_p) - \mu_i}^2 \geq {\epsilon''}/2 \right) 
\;\leq\; 2\exp\!\left( -\frac{{{\tau_p} \ \epsilon''}}{4\big( \sigma_j^2 + \tfrac{\sqrt{2}}{6}{\bar \eta}\sqrt{\epsilon''} \big)} \right).
\]

Choosing $\epsilon$ such that
$\dfrac{\sigma_j^2}{\epsilon} > \dfrac{\sqrt{2}{\bar \eta}}{6\sqrt{\epsilon}} \ \implies \text{\as{or equivalently}} \epsilon < \dfrac{18 \sigma_j^4}{{\bar \eta}^2}$ \as{this is not even a complete sentence!! Run a check to make sure nothing sounds off}.\\
Additionally we have $\epsilon'' \leq \dfrac{8}{\tau_p} \sigma_j^2 \log(12/\delta)$ from $\delta/6 \geq 2\exp\left(\dfrac{\tau_p \ \epsilon''}{4\sigma_j^2}\right)$.\\
Thus with $\epsilon'' <  \dfrac{18\sigma_j^4}{\bar \eta^2}$ and a union bound over all sources $j \in [M]$ we have with probability $\left(1-\dfrac{\delta}{6}\right)$,
\[
\Pr\!\left( \abs{ \hat \mu_{i}(\tau_p) - \mu_i}^2 \leq \dfrac{4}{\tau_p}\sigma_j^2 \log(12M/\delta) \right). 
\;
\]

\paragraph{Putting it all together: \as{why suddenly bold?}}
\begin{align*}
 \abs{\hat \sigma_{\tau_p}^2(j) - \sigma_j^2} \leq  &\abs{ \frac{2}{\tau_p} \sum_{k=1}^{\tau_p}  \big( X_k - \mu_{i})^2 - \sigma_j^2 }  + \abs{\frac{2}{\tau_p} \sum_{k=1}^{\tau_p} (\mu_i -  \hat \mu_{i}(\tau_p) \big)^2},
 \\
 \leq & \underbrace{4{\bar \eta}\sigma_j\sqrt{\dfrac{\log(12M/\delta)}{\tau_p}}}_{C} + \underbrace{\dfrac{8}{\tau_p}\sigma_j^2 \log(12M/\delta)}_D .
\end{align*} 

Now we can choose the number of times each source is queried (i.e, preprocessing budget $\tau_p$) such that $C > D$ and then bound the expression $C+D$ by $2C$. This is achieved when
$\tau_p > \dfrac{4\sigma_j^2 \log(12M/\delta)}{{\bar \eta}^2}$.\\
Now, we know from our problem setup and \cref{ass:bounded} that $\sigma_j^2 \leq \bar \eta^2$ for any $j \in [M]$. Therefore the greatest lower bound for $\tau_p$ is $\tau_p > 4\log(12M/\delta)$ which is true for any valid $M$. Additionally this constraint on $\tau_p$ is subsumed by the stopping condition stated in \cref{thm:taup_bound}.
Hence we have with probability $1-\dfrac{\delta}{3}$:
 \[
 \abs{\hat \sigma_{j, \mathrm{pre}}^2(\tau_p) - \sigma_j^2} \leq 8{\bar \eta}\sigma_j\sqrt{\dfrac{\log(12M/\delta)}{\tau_p}},
 \]
 \as{which concludes the proof. Add this to other proofs as appropriate.}
\end{proof}

\subsection{Proof of Stopping Condition of \algprec}
\HighVarianceElimination*
\label{pf:elimCondn}
\begin{proof}
For the purposes of this proof, we will be using the variance concentration defined above in \cref{lem:preprocVarConc} \as{please check punctuation thoroughly throughout the paper, this is in many places! If needed, take a printout and read the draft, correct with pen all at once and then correct them on Overleaf}\\
For the preprocessing algorithm to proceed without stopping, we need to find $\tau_p$ such that,

\as{please, don't write math like below, left indent them and connect them through $\implies$ }

\as{I think all $\sigma$ superscripts need to be $\sigma^2$?? right? we never did any conc on stds!}

\begin{align*}
\text{LCB}_{\tau_p,\mathrm{pre}}^{\sigma}(j) &\leq \text{UCB}_{\tau_p, \mathrm{pre}}^{\sigma}(j^*),
\\
\implies \hat \sigma^2_{j,\mathrm{pre}}(\tau_p) - \text{conf}^\sigma_j(\tau_p) &\leq \hat \sigma^2_{j^*,\mathrm{pre}}(\tau_p) + \text{conf}^\sigma_{j*}(\tau_p),
\\
\sigma_j^2 - 2\text{conf}^\sigma_j(\tau_p) \leq
\hat \sigma^2_{j,\mathrm{pre}}(\tau_p) - \text{conf}^\sigma_j(\tau_p) &\leq \hat \sigma^2_{j^*,\mathrm{pre}}(\tau_p) + \text{conf}^\sigma_{j*}(\tau_p) \leq {\sigma^*}^2 + 2\text{conf}^\sigma_{j*}(\tau_p),
\\
\sigma_j^2 - 2\text{conf}^\sigma_j(\tau_p) &\leq {\sigma^*}^2 + 2\text{conf}^\sigma_{j*}(\tau_p),
\\
\sigma_j^2 - {\sigma^*}^2 &\leq  2\text{conf}^\sigma_j(\tau_p) + 2\text{conf}^\sigma_{j*}(\tau_p) \leq 4\text{conf}^\sigma_j(\tau_p),
\\
\Delta^{\sigma^2}_j &\leq 4\text{conf}^\sigma_j(\tau_p),
\\
\Delta^{\sigma^2}_j &\leq 32{\bar \eta}\sigma_j\sqrt{\dfrac{\log(12M/\delta)}{\tau_p}} ~~[\text{Using \cref{lem:preprocVarConc}}].
\end{align*}
\as{$\implies$  missing}

We know that for any $j \in [M]$, we eliminate arm if $\sigma_j^2 - \sigma^2_* > {c^*}^2 \implies \Delta^{\sigma_j^2}_j > {c^*}^2$. 
\begin{align*}
{c^*}^2 &\leq 32{\bar \eta}\sigma_j\sqrt{\dfrac{\log(12M/\delta)}{\tau_p}},
\\
{c^*}^2 &\leq  \dfrac{1024 {\bar \eta}^4 \log(12M/\delta)}{\tau_p}, ~~[\sigma_j^2 \leq \bar \eta^2 ~(\text{\cref{ass:bounded}})]
\\
\tau_p &\leq \dfrac{1024 {\bar \eta}^4 \log(12M/\delta)}{{c^*}^4}.
\end{align*}
Hence for \algprec \ to proceed with the elimination of sources we need $\tau_p > \dfrac{1024 {\bar \eta}^4 \log(12M/\delta)}{{c^*}^4}$.
\end{proof}
\newpage
\section{Proofs Related to the adaptive phase of \algmab: \as{Again Caps!}}
\subsection{Proof of the reward variance concentration: \as{Caps are all messed up, check whole Supp!}}\label{pf:sigConc}
\begin{restatable}[Source Variance Concentration]{lemma}{sigconc}\label{lem:sigConc2}
Assume there exists a source $ j \in S_{\cG}$ such that $\sigma_j^2 \ge {c^*}^2$ where ${c^*} \in \mathbb{R}_+$ is a user-defined parameter chosen on the scale of $\sigma_j$. Define
\[
Q_j \defeq 
\begin{cases}
\max\{\kappa_j,\nu\}, & \text{if $\kappa_j$ is known},\\[2mm]
\max\{\bar{\eta}^2 \,\hat{\sigma}_j^2(t),\nu\}, & \text{if $\kappa_j$ is unknown},
\end{cases}
\].
Recall that $\nu \in \mathbb{R}_+$ is a user-defined parameter chosen on the scale of the fourth central moment $\kappa_j$.
If the source $j$, is queried for atleast $m_j(t)$ iterations where
$
m_j(t) = K + \dfrac{4\bar{\eta}^4 \log(3MT/\delta)}{\nu}.
$
Then for any $t \in [T]$, with probability at least $1-\delta/3$,
\begin{equation*}
\big| \hat{\sigma}_j^2(t) - \sigma_j^2 \big|
\;\le\;
\sqrt{\frac{Q_j(t) \log(3MT/\delta)}{m_j(t)-K}} .
\end{equation*}
\end{restatable}

\begin{proof}
Suppose $T_k(j) = \{k\le t: j_k = j\}$ denote the number of times  a source $j \in S_{\cG}$ was selected upto time $t$. 
    Consider the sequence of random variables defined by:
    \[
        D_k = \left(X_k - \hmu_{i_kj}(t)\right)^2 - \dfrac{\sigma_j^2(n_{i_kj}(t)-1)}{n_{i_kj}(t)}  \quad \forall k\in T_t(j).  \text{\as{note rendering issues, macros are missing?!}}
    \]
We observe that $\{D_{k}\}_{k\in T_t(j)}$ forms a Martingale Difference Sequence (MDS) \as{ref??} with respect to the filtration sequence $\mathcal{F}_{k} = \sigma \left(\{i_s\}_{s = 1}^{k}\right)$ \as{punctuation}


\paragraph{Verifying the MDS property:}
Note that $\mathbb{E}\left[ D_k|\mathcal{F}_{k}\right] = 0$. This follows from the property that \[\mathbb{E}\left[(X_i - \hmu_{i_kj}(t))^2|\mathcal{F}_{k}\right] = \dfrac{\sigma_j^2 (n_{i_kj}(t) - 1)}{(n_{i_kj}(t))}.\]
Specifically:
\begin{align*}
    \mathbb{E}\left[ D_k|\mathcal{F}_{k}\right] &= \mathbb{E}\left[ ((X_k-\hmu_{i_kj}(t))^2|\mathcal{F}_{k})^2\right] - \dfrac{\sigma_j^2(n_{i_kj}(t)-1)}{n_{i_kj}(t)},
    \\
    & = \dfrac{\sigma_j^2(n_{i_kj}(t)-1)}{n_{i_kj}(t)} - \dfrac{\sigma_j^2(n_{i_kj}(t)-1)}{n_{i_kj}(t)} = 0.
\end{align*}

\paragraph{Bounding $D_k$:}
Additionally we note that $D_k$ is bounded i.e $D_k \leq 2\bar \eta^2$.
This follows as:
\begin{align*}
    D_k &= \left(X_k - \hmu_{i_kj}(t)\right)^2 - \dfrac{\sigma_j^2(n_{i_kj}(t)-1)}{n_{i_kj}(t)},   \qquad k\in T_t(i)
    \\
    & \le (X_k - \mu_{i_k})^2 + (\mu_{i_k} - \hmu_{i_kj}(t))^2,\\
    & \le 2\bar \eta ^2.
\end{align*}

 Considering these properties we can now leverage standard martingale inequalities:
\begin{restatable}{freedman}{freedmanineq}
\begin{theorem}[Freedman's inequality, \cite{raban2022lecturenotes}]\label{thm:freedman}
Let $\{(D_k,\mathcal{F}_k)\}$ be a martingale difference sequence such that
\begin{enumerate}
  \item $\mathbb{E}[D_k\mid\mathcal{F}_{k-1}]=0$
  \item $D_k\le b$.
\end{enumerate}
Then for all $\lambda\in(0,1/b)$ and $\delta\in(0,1)$,
\[
\mathbb{P}\!\left(
  \sum_{k=1}^{T}D_k
  \le
  \lambda\sum_{k=1}^{T}\mathbb{E}[D_k^2\mid\mathcal{F}_{k-1}]
  +\frac{\log(1/\delta)}{\lambda}\right)
\ge 1-\delta.
\]
\end{theorem}
\end{restatable}

Now, since the martingale difference sequence (MDS) $\{D_k\}_{k \in T_t(j)}$ satisfies the conditions of \cref{thm:freedman}, we can apply Freedman's inequality. 
We begin by evaluating the partial sum $\sum_{k \in T_t(j)} D_k$, which corresponds to the left-hand side of \cref{thm:freedman}.
 
\paragraph{Summing $D_k$:}
\begin{align*}
    \sum_{k \in T_t(j)} D_k &= \sum_{i=1}^K \sum_{\ell\in\tau_{ij}(t)} \left( \left(X_{\ell} - \hmu_{i_{\ell}j}(t)\right)^2 - \dfrac{\sigma_j^2(n_{i_{\ell}j}(t)-1)}{n_{i_{\ell}j}(t)}  \right), 
    \\
    &= \sum_{i=1}^K  \left( \sum_{\ell\in\tau_{ij}(t)}\left(X_{\ell} - \hmu_{i_{\ell}j}(t)\right)^2 - \ \sigma_j^2(n_{i_{\ell}j}(t)-1)  \right),
    \\
    &= (m_j(t)-K)\left|\left( \hsigma_{j}^2(t) - \sigma_j^2\right)\right|.
\end{align*}

Our next task is to bound the predictable quadratic variation term
$\sum_{k \in T_t(j)} \E\!\left[D_k^2 \mid \mathcal{F}_{k}\right],$
which serves as the conditional variance proxy in \cref{thm:freedman}.
Note that $\E[D_k \mid \mathcal{F}_{k}] = 0$ since $\{D_k\}$ is a martingale difference sequence.\\

To achieve this we will make use of \emph{Result~3} in \citep{8bb291e6-188d-3d78-b02a-67cdbdc76b17} (p.~284), which provides a moment expression for the (unbiased) sample variance under simple random sampling. Specifically, 
\begin{equation}\label{eq:varSampVariance}
\text{Var}(S_n^2) =  \dfrac{1}{n}\left( \mu_4 - \dfrac{n-3}{n-1}\sigma^4\right)    
\end{equation}
Here, $S_n^2 \defeq \frac{1}{n-1}\sum_{i=1}^n (X_i-\bar X)^2$ denotes the unbiased sample variance computed from $n$ i.i.d.\ samples $X_1,\dots,X_n$ with mean $\mu$ and population variance $\sigma^2$.  
Moreover, $\mu_4 \defeq \mathbb{E}[(X-\mu)^4]$ denotes the fourth central moment of the underlying distribution.

Adapting the above (\cref{eq:varSampVariance}) to our setting we obtain the following expression:
\begin{equation}\label{eq:varSEbound}
 \text{Var}(\hsigma^2_{ij}(t)) = \dfrac{1}{n_{ij}(t)}\left( \kappa_j- \dfrac{n_{ij}(t)-3}{n_{ij}(t)-1} \sigma_j^4\right) \le \dfrac{\kappa_j}{n_{ij}(t)}.  
\end{equation}

\paragraph{Bounding Variation Term:}
\begin{align*}
    \sum_{k \in T_t(j)}\mathbb{E}[D_k^2|\mathcal{F}_{k-1}] &= \sum_{i=1}^K \sum_{\ell \in \cT_{ij}(t)}  \mathbb{E} \left[\left( \left(X_{\ell} - \hmu_{i_{\ell}j}(t)\right)^2 - \dfrac{\sigma_j^2(n_{i_{\ell}j}(t)-1)}{n_{i_{\ell}j}(t)}  \right)^2  \right],
    \\
    &\leq \sum_{i=1}^K \sum_{\ell \in \cT_{ij}(t)} \text{Var} \left[   \left(X_k - \hmu_{i_kj}(t)\right)^2 \right],
    \\
    &\leq \sum_{i=1}^K \text{Var}(\hsigma^2_{ij} (t)) (n_{ij}(t)-1)^2,
    \\
    &\leq \sum_{i=1}^K \dfrac{(n_{ij}(t)-1)^2 \kappa_j}{n_{ij}(t)}, ~~[\text{From  \cref{eq:varSEbound}}] 
    \\
    &\leq (m_j(t) - K)\kappa_j,
    \\
     &\leq (m_j(t) - K)Q_j(t).   
\end{align*}
where  \[
Q_j(t)=
\begin{cases}
\max \{ \kappa_j, \nu \} & \text{if } \kappa_j  \text{ is  known}\\
\max \{ \bar \eta^2 \hat \sigma_j^2(t), \nu \} & \text{if } \kappa_j  \text{ is  unknown}
\end{cases}
\]
\paragraph{Putting it all together:}
Hence, we have with high probability $1-\delta/3$
\[
(m_j(t)-K)\left|\left( \hsigma_{j}^2(t) - \sigma_j^2\right)\right| \leq \lambda (m_j(t) - K) Q_j(t) \ + \ \dfrac{\log(3/\delta)}{\lambda}. 
\]
Optimizing over $\lambda$ (i.e., minimizing the resulting bound) yields \[
\lambda^* = \sqrt{\dfrac{\log(3/\delta)}{Q_j(t)(m_j(t)-K)}}.
\]

Applying the Freedman boundary constraint over $\lambda^*$ yields:
\begin{align*}
    \lambda^* = \sqrt{\dfrac{\log(3MT/\delta)}{Q_j(t)(m_j(t)-K)}} &< \dfrac{1}{2\bar \eta^2},
    \\
    m_j(t) &> K + \dfrac{4\bar \eta^4 \log(3MT/\delta)}{Q_j(t)},
    \\
     m_j(t) &> K + \dfrac{4\bar \eta^4 \log(3MT/\delta)}{\nu}. 
\end{align*}
The last line follows from the definition of $Q_j(t)$. Applying a union bound over all sources $j \in S_{\mathcal{G}} \subseteq [M]$ and all $t \in [T]$ yields the following bound that holds with probability at least $1-\delta/3$.

\[
\left| \hsigma_{j}^2(t) - \sigma_j^2\right| \leq 2\sqrt{\dfrac{Q_j(t) \log(3MT/\delta)}{m_j(t) - K}}.
\]

For the above bound to hold, it suffices that
\begin{equation*}\label{eq:sigconc-constraint}
m_j(t) > K + \frac{4\bar{\eta}^4 \log(3MT/\delta)}{\nu}.
\end{equation*}
\end{proof}

\subsection{Proof of the Sandwiching Corollary}\label{pf:swcorr}
\VarianceSandwich*
\begin{proof}
 Recall from \cref{lem:sigConc} that:  $
    \left| \hsigma_{j}^2(t) - \sigma_j^2\right| \leq 2\sqrt{\dfrac{Q_j(t) \log(3MT/\delta)}{m_j(t) - K}}.
$\\ 
Choosing $m_j(j)$ such that $m_j(t) = K + \dfrac {16 \ Q_j(t)  \log(3MT/\delta)}{\sigma_j^4}$ we have\\
\begin{align*}
\abs{\hat \sigma_j^2(t) - \sigma_j^2} \leq \dfrac{\sigma_j^2}{2} \implies\dfrac{\sigma_j^2}{2} \leq {\hat \sigma_j^2(t) } \leq \dfrac{3 \sigma_j^2}{2},
\\
\implies \sigma_j^2 \leq {2\hat \sigma_j^2(t) } \leq 3 \sigma_j^2.
\end{align*}
Now $m_j(t) = K + \dfrac {16Q_j(t)  \log(3MT/\delta)}{\sigma_j^4}$ can be upper bounded by $K + \dfrac{16 \bar \eta^4 \log(3MT/\delta)}{{c^*}^4}$ using our regime condition ($\exists \ j \in S_{\cG} \subseteq [M]$ such that $\sigma_j^2 \ge {c^*}^2$ ) and \cref{ass:bounded}.
\end{proof}

\subsection{Proof of the mean reward concentration:}\label{pf:muConc}
\muconc*
\begin{proof}
Suppose $T_t(i) = \{s \leq t : i_s = i\}$ denotes the set of time steps at which arm $i$ was selected up to time $t$. Consider the sequence of random variables defined by
\[
D_{t,i} \defeq X_{t,i} - \mu_i, \quad \text{for all } t \in T_t(i).
\]

To establish a concentration bound, the key observation is that $\{D_{t,i}\}_{t \in T_t(i)}$ forms a martingale difference sequence with respect to the filtration 
$
\mathcal{F}_{t-1} = \sigma\left(\{X_s, i_s, j_s\}_{s=1}^{t-1}, i_t = i\right),
$
generated by all arm selections and observations prior to round $t$. Precisely note that each $D_{t,i}$ is integrable, adapted to $\mathcal{F}_{t}$, and satisfies
$
\E[D_{t,i}\mid\mathcal{F}_{t-1}] 
= \E[X_{t,i} - \mu_i \mid \mathcal{F}_{t-1}]
= \mu_i - \mu_i = 0.
$
Thus, by construction, the sequence satisfies the martingale difference property and enables us to leverage standard martingale concentration inequalities.

\freedmanineq*

We have a bound $\bar \eta$ on $D_{t,i}$, the first requirement of the theorem has already been proven. Applying Freedman's, we get

\begin{align*}
&\mathbb{P}\left(\sum_{t \in T_i(t)} D_{t,i} \leq \lambda \sum_{t \in T_i(t)} \mathbb{E}[D_{t,i}^2 | \mathcal{F}_{t-1}] + \frac{\log(3/\delta)}{\lambda}\right) \geq 1 - \dfrac{\delta}{3}
,\\
&\mathbb{P}\left(\sum_{t \in T_i(t)} D_{t,i} \leq \lambda \sum_{j=1}^M n_{ij}(t)\cdot\sigma_{j}^2 + \frac{\log(3/\delta)}{\lambda}\right) \geq 1 - \dfrac{\delta}{3}. ~~\left[\mathbb{E}[D_{t,i}^2 | \mathcal{F}_{t-1}] = \sigma_{j}^2 \right] 
\\
\end{align*}
Optimizing over $\lambda$ gives us 
$
\lambda^* = \sqrt{\dfrac{\log(3/\delta)}{\sum_{j=1}^M n_{ij}(t)\cdot\sigma_{j}^2}}.
$

\[
\implies \mathbb{P}\left(\sum_{t \in T_i(t)} D_{t,i} \leq 2\sqrt{\log \left( \frac{3}{\delta} \right) \sum_{j=1}^{M}n_{ij}(t) \sigma_{j}^2} \right) \geq 1 - \dfrac{\delta}{3}.
\]
A union bound across K arms and T time-steps gives us
\[
\implies \mathbb{P}\left(\sum_{t \in T_i(t)} D_{t,i} \leq 2\sqrt{\log \left( \frac{3KT}{\delta} \right) \sum_{j=1}^{M}n_{ij}(t) \sigma_{j}^2} \right) \geq 1 - \dfrac{\delta}{3}.
\]
Dividing both sides by $n_{i}(t)$ we get obtain concentration bound with probability $1-\dfrac{\delta}{3}$
\[
\abs{\hat{\mu}_t(i) - \mu_i} \leq \dfrac{2\sqrt{ \log(3KT/\delta) \sum_{j=1}^M n_{ij}(t)\sigma_j^2 }}{n_{i}(t)}.
\]

To conclude the proof of \cref{lem:muConc}, we must satisfy the Freedman boundary constraint. Specifically we need to make sure $n_{i}(t)$ is large enough so $\lambda^* < \dfrac{1}{\bar \eta}$
\[
\lambda^* = \sqrt{\frac{\log(3KT/\delta)}{\sum_{j=1}^n n_{ij}(t)\cdot\sigma_{j}^2}} < \dfrac{1}{\bar \eta}.
\]

\paragraph{Variance Identification (VarI):}
If a source $j \in S_{\mathcal{G}}$ satisfies $\sigma_j^2 \geq {c^*}^2$, then we can derive an empirical test to identify such sources with high probability.
Let $\tau_{\text{VarI}}$ denote the number of queries required for this identification.
Recall from the preprocessing phase that, once source $j$ has been queried at least $\tau_{\text{VarI}}$ times, we can compute the LCB on the variance estimate using \cref{eq:pp_lcb_var}. Specifically

\begin{align*}
\text{LCB}^{\sigma, \mathrm{pre}}_{\tau_{\text{VarI}}}(j) &=  \hat \sigma_{j}^2(\tau_{\text{VarI}}) - 8{\bar \eta}^2 \sqrt{\dfrac{\log(12M/\delta)}{\tau_{\text{VarI}}}},
\\
&\geq \sigma_j^2 -  16\bar\eta^2 \sqrt{\dfrac{\log(12M/\delta)}{\tau_{\text{VarI}}}}, ~~[\text{Using \cref{lem:preprocVarConc}}]
\\
&\geq {c^*}^2 - 16 \bar\eta^2 \sqrt{\dfrac{\log(12M/\delta)}{\tau_{\text{VarI}}}}, ~~[\text{As }\sigma_j^2 > {c^*}^2]
\\
&\geq {c^*}^2 - \dfrac{{c^*}^2}{2}, ~~[\text{Choosing } \tau_{\text{VarI}} = \dfrac{1024 \ \bar \eta^4 \log(12M/\delta)}{{c^*}^4}]
\\
&\geq \dfrac{{c^*}^2}{2}.
\end{align*}
Hence if  a source has $\text{LCB}^{\sigma, \mathrm{pre}}_{\tau_{\text{VarI}}}(j) \geq \dfrac{{c^*}^2}{2}$ after being queried for $\tau_{\text{VarI}}$ iterations, then we can conclude that its true variance $\sigma_j^2$, satisfies $\sigma_j^2 \ge {c^*}^2$ with high probability.\\
Finally, note that from \cref{thm:taup_bound} $\tau_P = O(\tau_{\text{VarI}})$, so $\tau_P$ subsumes $\tau_{\text{VarI}}$. Henceforth, we only use $\tau_P$.

\paragraph{Required starting conditions:}
Consider the denominator of $\lambda^*$. Let $\tilde j$ be a source that satisfies our \textbf{VarI} check. Now
\[
\sum_{j=1}^n n_{ij}(t)\cdot\sigma_{j}^2 \geq n_{i \tilde j}(t) \cdot \sigma_{\tilde j}^2 \geq \alpha {c^*}^2.
\]
where $\alpha$ is the number of times we must query an arm with source $\tilde j$.
Hence 
\begin{align*}
    \sqrt{\frac{\log(3KT/\delta)}{\alpha c^*}} &< \dfrac{1}{\bar \eta} \implies
    \alpha > \dfrac{\bar \eta^2 \log(3KT/\delta)}{{c^*}^2}.
\end{align*}
which is the required condition for the mean reward concentration bound to hold
\end{proof}

\newpage
\subsection{Validity of the UCB/LCB Constructions}\label{pf:coverage}
\begin{restatable}[Mean Reward UCB]{corollary}{mucb}
\label{cor:ucb}
Consider any $\delta \in (0,1)$. 
At any time step $t \in [T]$ and arm $i \in [K]$, with probability at least $1-\dfrac{\delta}{3}$: 
    \begin{align*}
        \mu_i \leq \ucb_t^\mu(i).
    \end{align*}
\end{restatable}
\begin{proof}
The proof of the above result directly follows from the mean reward concentration defined in \cref{lem:muConc}, our definition of $\ucb_t^\mu(i)$ from \cref{eq:ucb_mu1} and an application of \cref{cor:varsandwich}.
\end{proof}

\begin{restatable}[Reward Variance LCB]{corollary}{varlcb}
\label{cor:varlcb}
Consider any $\delta \in (0,1)$. 
At any time step $t \in [T]$ and source $j \in [M]$, with probability at least $1-\dfrac{\delta}{3}$: 
    \begin{align*}
        \sigma_j \geq \lcb_t^\sigma(j).
    \end{align*}
\end{restatable}
\begin{proof}
The proof of the above result directly follows from the mean reward concentration defined in \cref{lem:sigConc}, our definition of $\lcb_t^\mu(i)$ from \cref{eq:lcb_var} and an application of \cref{cor:varsandwich}.
\end{proof}
\newpage
\section{Regret Analysis:}
\subsection{Proof of \cref{thm:regfin}} \label{pf:main_regret}
\thmregfin*

\as{I am afraid the bounds are quite off all in all: $Q_j(t)$ does not make sense at all, probably you need $T$}

\as{I think $\sqrt{K^2c^{*2}M}$ term is not needed in the O() notation, if we choose $c^* \leq {1}$ (assuming $\bar \mu \geq 1$ at least), right?}

\as{why the 1st and the 4th term in the regret is exactly same??!! we need to out more effort simplifying the 1st 4 terms in the O() notation, instead of leaving it for the readers to figure out the tradeoff}

\as{Again the 1st and the 4th term in the regret is exactly same in the WC regret too}

\noindent \textbf{Proof:}
\as{write a intro line, no abrupt/akward start...}

\paragraph{Exploration Regret:}
We begin by stating the regret accumulated by preprocessing and initial exploration:
\begin{enumerate}
    \item Preprocessing: Preprocessing accumulates regret of order $M\tau_p$, where
        $\tau_p \defeq \dfrac{1024 \bar \eta^4 \log(12M/\delta)}{{c^*}^4}$.  
    \item Initial exploration: The initial exploration phase accumulates regret of order $K\alpha + \tilde{M}\beta$, where $\tilde{M} \defeq |S_{\mathcal{G}}|$ and
    \[
        \alpha = \frac{\bar \eta^2 \log(3KT/\delta)}{{c^*}^2},
        \qquad
        \beta = 2K + \frac{4 \bar{\eta}^4 \log(3MT/\delta)}{\nu}
        + \frac{16\bar \eta^4 \log(3MT/\delta)}{{c^*}^4}.
    \]
\end{enumerate}

Hence, the regret incurred from exploration is:
\begin{align*}
\text{Reg}^{\text{Exp}}_T &= M\tau_p \bar \mu + K\alpha \bar \mu + M\beta \bar \mu 
\\
 &=    \frac{M \cdot 1024 \bar \eta^4 \log(12M/\delta)}{{c^*}^4} \bar \mu +  \frac{K\bar \eta^2 \log(3KT/\delta)}{{c^*}^2} \bar \mu  \space \\
 &+ M\left(2K + \frac{4 \bar{\eta}^4 \log(3MT/\delta)}{\nu}  + \frac{16 \bar \eta^4 \log(3MT/\delta)}{{c^*}^4}\right)\bar \mu.
\end{align*}

The order of $\text{Reg}^{\text{Exp}}_T$ is 
\begin{align*}
\tilde O(\text{Reg}^{\text{Exp}}_T) 
= \ \tilde O\left( \dfrac{ M\bar \eta^4 \bar \mu}{{c^*}^4} + {\frac{K\bar \eta^2 \bar \mu}{{c^*}^2}}
+ {KM \bar \mu + \frac{M\bar \eta^4 \bar \mu}{\nu} + \frac{M \bar \eta^4 \bar \mu}{{c^*}^4}} \right).
\end{align*}

We now seek to bound the regret for the remaining $T - M\tau_p - K\alpha -\tilde M \beta$ rounds of \algmab \. To achieve this we must first bound $n_i(t)$ and $m_j(t)$.
\paragraph{Bounding $n_i(t)$:}
Note that at any time $t$, arm-$i$ can not be selected if $\ucb_t^\mu(i) \leq \ucb_t^\mu(i^*)$. So arm-$i$ only gets selected at time $t$ if:
\begin{align*}
 \mu_{i^*} & \leq \ucb_t^\mu(i^*) \leq \ucb_t^\mu(i) = \hat{\mu}_i(t) + \dfrac{2\sqrt{ \log(3KT/\delta) \sum_{j=1}^{\tilde M} n_{ij}(t)\cdot2\hat\sigma_j^2(t) }}{n_i(t)}, ~~\left[ \text{ From \cref{cor:varsandwich}}\right] 
\\
\mu_{i^*} &  \leq {\mu}_i + \dfrac{4\sqrt{ \log(3KT/\delta) \sum_{j=1}^M n_{ij}(t)\cdot(3\sigma_j^2)  }}{n_i(t)}, ~~[\text{Using (\cref{lem:muConc})} ]
\\
\implies n_i(t) & \leq \dfrac{4\sqrt{3 \log(3KT/\delta) \sum_{j=1}^M n_{ij}(t)\cdot \sigma_j^2  }}{\Delta_i}. 
\end{align*}

\paragraph{Bounding $m_j(t)$:}
We also know that at any time $t$, a non optimal source $j$ gets selected if, $ \lcb_t^\sigma(j) \leq \lcb_t^\sigma(j^*)$  
\begin{align*}
&  \sigma_j^2 \leq \ucb_t^\sigma(j),
\\
\implies & \sigma_j^2 \leq  \hat \sigma_j(t)^2 + \text{conf}_t(j) ,
\\
 & \sigma_j^2 - 2\cdot \text{conf}_t(j) \leq  \hat \sigma_j(t)^2 - 2\cdot \text{conf}_t(j)+ \text{conf}_t(j),
\\
 & \sigma_j^2 - 2\cdot \text{conf}_t(j) \leq  \hat \sigma_j(t)^2 - \text{conf}_t(j) \leq \lcb_t^\sigma(j), 
\\
& \sigma_j^2 - 2\cdot\text{conf}_t(j) \leq  \lcb_t^\sigma(j) \leq \lcb_t^\sigma(j^*) \leq \sigma^{*2},
\\
& \sigma_j^2 - \sigma^{*2} \leq 2\cdot\text{conf}_t(j), 
\\
\implies & \Delta_j^{\sigma^2} \leq 2\cdot\text{conf}_t(j). 
\end{align*}
Where $\Delta_j^{\sigma^2}$ is defined as $\sigma_j^2 - {\sigma^*}^2$\\
From our variance concentration defined in {\cref{lem:sigConc}}
\[
\text{conf}_t(j) = 2\sqrt{\dfrac{Q_j(t) \log(3MT/\delta)}{m_j(t) - K}}.
\]
Hence we have
\begin{align*}
    \Delta_j^{\sigma^2} &\leq 4\sqrt{\dfrac{Q_j \log(3MT/\delta)}{m_j(t) - K}},
    \\
    m_j(t)-K &\le\dfrac{16  Q_j(t) \log(3MT/\delta)}{(\Delta_j^{\sigma^2})^2},
    \\
    m_j(t) &\le K + \dfrac{16 Q_j(t) \log(3MT/\delta)}{(\Delta_j^{\sigma^2})^2}.
\end{align*}

Therefore, for our algorithm to operate normally i.e, select the optimal source we must query each source at least:
\begin{equation} \label{eq:mjtBound}
    m_j(t) > K + \dfrac{16 Q_j(t) \log(3MT/\delta)}{(\Delta_j^{\sigma^2})^2}.
\end{equation}

\paragraph{Deriving regret for the remaining rounds:} The Regret in $T- M\tau_p - K\alpha - \tilde M\beta$ rounds of Running \algmab:
\begin{align*}
    \text{Reg}_T  = \sum_{i = 2}^K n_i(T) \Delta_i &\leq  \sum_{i=2}^K 4\sqrt{3 \log(3KT/\delta) \sum_{j=1}^M n_{ij}(t)\cdot\sigma_j^2},
    \\
    \dfrac{\text{Reg}_T}{4\sqrt{3\log(3KT/\delta)}} &{\leq} \sum_{i=2}^K\sqrt{ \sum_{j|\sigma_j- {\sigma^*} > \gamma} n_{ij}(T) \ \sigma_j^2 + \sum_{j|\sigma_j - {\sigma^*} \leq \gamma} n_{ij}(T) \ \sigma_j^2 } ~~[\text{split surviving variances across $\gamma$}],
    \\
     &{\leq} \sum_{i=2}^K \left(\sqrt{ \sum_{j|\sigma_j - {\sigma^*} > \gamma} n_{ij}(T) \ \sigma_j^2} + \sqrt{\sum_{j|\sigma_j - {\sigma^*} \leq \gamma} n_{ij}(T) \ \sigma_j^2 }\right)  ~~[\sqrt{a+b} \leq \sqrt{a} + \sqrt{b}],
    \\
      &{\leq} \sum_{i=2}^K  \left( \sqrt{ \sum_{j|\sigma_j - {\sigma^*} > \gamma} n_{ij}(T) \ (\sigma_j - \sigma^* + \sigma^*)^2} + \sqrt{\sum_{j|\sigma_j - {\sigma^*} \leq \gamma} n_{ij}(T) \ \sigma_j^2 } \right)  ~~[\pm \sigma^*], 
    \\
     &\leq \underbrace{\sum_{i=2}^K\sqrt{ \sum_{j|\sigma_j - {\sigma^*} > \gamma} n_{ij}(T) \ (\sigma_j - \sigma^*)^2}   }_{A} +  \underbrace{\sum_{i=2}^K\sqrt{ \sum_{j|\sigma_j - {\sigma^*} > \gamma} n_{ij}(T) \cdot {\sigma^*}^2}}_{B} \\
     &+ \sum_{i=2}^K\underbrace{\sqrt{\sum_{j|\sigma_j - {\sigma^*} \leq \gamma} n_{ij}(T) \ \sigma_j^2 }}_{C}.  
\end{align*}
where the last step follows since $[(a+b)^2 \le 2a^2 + 2b^2]$. We will bound term A, term B and term C separately, as shown below:

\paragraph{i) Bounding Term A:}
\begin{align*}
    \text{A} &= (K-1) \sum_{i=2}^K\dfrac{1}{K-1}\sqrt{\sum_{j|\sigma_j - {\sigma^*} > \gamma}n_{ij}(T) (\sigma_j- {\sigma^*})^2},
    \\
    &\leq \sqrt{(K-1)} \sqrt{\sum_{i=2}^K \sum_{j|\sigma_j - {\sigma^*} > \gamma} n_{ij}(T) (\sigma_j-{\sigma^*})^2} \quad ~~[\text{Using Cauchy-Schwarz inequality \citep{steele2004cauchy}}],
    \\
    &\leq \sqrt{K} \sqrt{ \sum_{j|\sigma_j - {\sigma^*} > \gamma} m_T(j) (\sigma_j-{\sigma^*})^2},
    \\
    &\le \sqrt{K} \sqrt{\sum_{j|\sigma_j - {\sigma^*} > \gamma} \left(K + \dfrac{16 \ Q_j(t) \log(3MT/\delta)}{(\Delta_j^{\sigma^2})^2} \right) (\sigma_j-{\sigma^*})^2} \quad ~~[\text{From \ref{eq:mjtBound}}],
    \\
    &\le \sqrt{K} \sqrt{\sum_{j|\sigma_j - {\sigma^*} > \gamma} K(\sigma_j-\sigma^*)^2 + \sum_{j|\sigma_j - {\sigma^*} > \gamma} \dfrac{16 \ Q_j(t) \log(3MT/\delta)}{(\sigma_j-\sigma^*)^2(\sigma_j + \sigma^*)^2}  (\sigma_j-{\sigma^*})^2},
    \\
    &\le \sqrt{K} \sqrt{\sum_{j|\sigma_j - {\sigma^*} > \gamma} K\left(\sqrt{{\sigma^*}^2 + {c^*}^2 } -\sigma^*\right)^2 + \sum_{j|\sigma_j - {\sigma^*} > \gamma}\dfrac{16 \ Q_j(t) \log(3MT/\delta)}{(\sigma_j + \sigma^*)^2}} 
    ~~\left[\sigma_j < \sqrt{{\sigma^*}^2 + {c^*}^2 } \right], 
    \\
    &\le \sqrt{K} \sqrt{\sum_{j|\sigma_j - {\sigma^*} > \gamma} K{c^*}^2  + \sum_{j|\sigma_j - {\sigma^*} > \gamma}\dfrac{16 \ Q_j(t) \log(3MT/\delta)}{(\sigma_j + \sigma^*)^2}} ~~[\sqrt{a+b} \leq \sqrt{a} + \sqrt{b}],
    \\
    &\le \sqrt{K} \sqrt{ KM{c^*}^2  + \sum_{j|\sigma_j - {\sigma^*} > \gamma}\dfrac{16 \ Q_j(t) \log(3MT/\delta)}{(\gamma + 2\sigma^*)^2}} ~~[\text{Using $\gamma$ constraint from the sum}].
    \\
\end{align*}

The final bound on term A is: $A  \le \sqrt{K} \sqrt{ KM{c^*}^2  + \sum\limits_{j|\sigma_j - {\sigma^*} > \gamma}\dfrac{16 \ Q_j(t) \log(3MT/\delta)}{(\gamma + 2\sigma^*)^2}}$

\paragraph{ii) Bounding Term B:}
\begin{align*}
B &=\sum_{i=2}^K\sqrt{ \sum_{j|\sigma_j - {\sigma^*} \leq \gamma} n_{ij}(T) \cdot {\sigma_j}^2}
 \leq \sum_{i=2}^K\sqrt{ \sum_{j=1}^M n_{ij}(T) \cdot {\sigma^*}^2},
\\
 & \leq \sigma^* \sum_{i = 2}^K \frac{\sqrt{ \sum_{j=1}^M n_{ij}(T) 2L \ \sigma^* }}{\sqrt \Delta_i}\frac{\sqrt{\Delta_i} }{\sqrt{2L\sigma^*}},   
\\ 
& \leq \sigma^* \sqrt{\sum_{i = 2}^K \frac{2L\sigma^*}{\Delta_i} } \sqrt{ \sum_{i = 2}^K \sum_{j=1}^M \frac{n_{ij}(T) \Delta_i}{2L\sigma^*} },
\\
& \leq \frac{\sigma^*}{2} \biggsn{ {\sum_{i = 2}^K \frac{2L\sigma^*}{\Delta_i} } + { \sum_{i = 2}^K \sum_{j=1}^M \frac{n_{ij}(T) \Delta_i}{2L\sigma^*} } }
= \sum_{i = 2}^K \frac{L\sigma^{*2}}{\Delta_i} + \frac{\text{Reg}_T}{4L} ~~\left[\text{where } L = 4\sqrt{6 \log (3KT/\delta)}\right].
\end{align*}

The final bound on term B is: $B \le \sum_{i = 2}^K \frac{L\sigma^{*2}}{2\Delta_i} + \frac{ \text{Reg}_T}{4L}$
where $L = 4\sqrt{6 \log (3KT/\delta)}$. 

\paragraph{Bounding Term C:}
\begin{align*}
C &=\sum_{i=2}^K\sqrt{\sum_{j|\sigma_j - {\sigma^*} \leq \gamma} n_{ij}(T) \ \sigma_j^2 } = \sum_{i=2}^K\sqrt{\sum_{j|\sigma_j - {\sigma^*} \leq \gamma} n_{ij}(T) \ ({\gamma + \sigma^*})^2 }.
\end{align*}
Proceeding in a similar fashion to how we bounded term B, we get\\
The final bound on term C is: $C \le \sum\limits_{i = 2}^K \dfrac{L(\sigma^*+\gamma)^{2}}{\Delta_i} + \dfrac{ \text{Reg}_T}{4L}$
where $L = 4\sqrt{6 \log (3KT/\delta)}$. 

\paragraph{Putting it all together:}
Hence the regret in $T - M\tau_p - K\alpha - \tilde M\beta$ rounds of \algmab:
\begin{align*}
 \text{Reg}_T &  \le 2\sqrt{K} \sqrt{ KM{c^*}^2  + \sum_{j|\sigma_j - {\sigma^*} > \gamma}\dfrac{16 \ Q_j(t) \log(3MT/\delta)}{(\gamma + 2\sigma^*)^2}} + \sum_{i = 2}^K \frac{L\sigma^{*2}}{\Delta_i}+ \sum_{i = 2}^K \frac{L(\sigma^*+\gamma)^{2}}{\Delta_i},
 \\
 &  \le 2 \sqrt{ K^2M{c^*}^2  + \sum_{j|\sigma_j - {\sigma^*} > \gamma}\dfrac{16 \ KQ_j(t) \log(3MT/\delta)}{(\gamma + 2\sigma^*)^2}} + 2\sum_{i = 2}^K \frac{L(\sigma^*+\gamma)^{2}} {\Delta_i}~~[\text{as } \gamma \in \mathbb{R}_+].
\end{align*}

\paragraph{Total Regret:}
The total regret of \algmab  \ = \  $\text{Reg}^{\text{Exp}}_T+\text{Reg}_T$
\begin{align*}
\text{Reg}^{\text{\algmab}}_T = M\tau_p \bar \mu \ + \  K\alpha \bar \mu \ + \  M\beta \bar \mu \ + &\  2 \sqrt{ K^2M{c^*}^2  + \sum_{j|\sigma_j - {\sigma^*} > \gamma}\dfrac{16 \ KQ_j(t) \log(3MT/\delta)}{(\gamma + 2\sigma^*)^2}}  \\
&+ \sum_{i = 2}^K \frac{2L(\sigma^*+\gamma)^{2}}{\Delta_i}. 
\end{align*}
\as{punctuation issue!}
where \as{recall from \cref{which section??}  that we set:}
\as{no bullet list, just in state in a line, looking akward}

\begin{enumerate}
    \item $\tau_p = \dfrac{1024 \bar \eta^4 \log(12M/\delta)}{{c^*}^4}$. \as{$\geq$}
    \item $\alpha = \dfrac{\bar \eta^2 \log(3KT/\delta)}{{c^*}^2}$. \as{are they exactly equal? verify}
    \item $\beta =   2K + \dfrac{4 \bar{\eta}^4 \log(3MT/\delta)}{\nu} + \dfrac{16 \bar \eta^4 \log(3MT/\delta)}{{c^*}^4}$.
    \item $L = 4\sqrt{6 \log (3KT/\delta)}$.
\end{enumerate}

\paragraph{Order of Regret:} \as{write a full complete line, we don't need a paragraph for this: Combining everything, finally this leads to a regret bound of ....}

\begin{align*}
\tilde O(\text{Reg}^{\text{\algmab}}_T) =& \tilde O\left( \dfrac{ M\bar \eta^4 \bar \mu}{{c^*}^4} + {\frac{K\bar \eta^2 \bar \mu}{{c^*}^2}}
+ {KM \bar \mu + \frac{M\bar \eta^4 \bar \mu}{\nu} + \frac{M \bar \eta^4 \bar \mu}{{c^*}^4}}\right) \\ 
&+ \tilde O\left(  \sqrt{ K^2M{c^*}^2  + \sum_{j|\sigma_j - {\sigma^*} > \gamma}\dfrac{16 \ KQ_j(t) }{(\gamma + 2\sigma^*)^2}} + \sum_{i = 2}^K \frac{(\sigma^*+\gamma)^{2}}{\Delta_i}\right).
\end{align*}

\paragraph{Instance Independent \as{(WC)} Regret:}
This proves the first part of \cref{thm:regfin}, which yields an instance-dependent guarantee of \algmab.
\noindent Finally, to see the worst case (instance independent) regret analysis of \algmab, note that the regret \as{in the final} $(T-M\tau_p - K\alpha - \tilde M\beta)$ rounds of \algmab\, can be alternatively bounded as:
\begin{align*}
    \text{Reg}_T & = \sum_{i = 2}^K n_i(T) \Delta_i \leq  \sum_{i=2}^K 4\sqrt{3 \log(3KT/\delta) \sum_{j=1}^M n_{ij}(T)\sigma_j^2  },
    \\
    &  \leq \sqrt{{\sigma^*}^2 + {c^*}^2} \sum_{i=2}^K 4\sqrt{ 3\log(3KT/\delta) \sum_{j=1}^M n_{ij}(T)  } ~~(\sigma_j \leq \sqrt{{\sigma^*}^2 + {c^*}^2}),
    \\
    & \leq 4 \sqrt{{\sigma^*}^2 + {c^*}^2} \sqrt{3K \log(3KT/\delta) }\sqrt{ \sum_{i=2}^K \sum_{j=1}^M n_{ij}(T)  } ~~(\text{applying Cauchy's Schwarz}),
    \\
    & \leq 4(\sigma^* + c^*) \sqrt{3KT \log(3KT/\delta) } ~~(\sqrt{a+b} \leq \sqrt{a} + \sqrt{b}),
    \\
    & = 4(\sigma^* + c^*) \sqrt{3K T \log(3KT/\delta) } \text{\as{fullstop missing}}
\end{align*}

{\as{repeated line at the end!! You also need to adjust $\delta$ --- say/ remind how we set it. This paper needs a thorough proofreading; again, take a printout if needed and read line by line. Lmk if you need a printout.}}

Therefore the order the instance independet regret is
\begin{align*}
&\tilde O(\text{Reg}_T^{\text{Exp}} + (\sigma^* + c^*) \sqrt{KT}\\
\equiv \ &\tilde O\left( \dfrac{ M\bar \eta^4 \bar \mu}{{c^*}^4} + {\frac{K\bar \eta^2 \bar \mu}{{c^*}^2}}
+ {KM \bar \mu + \frac{M\bar \eta^4 \bar \mu}{\nu} + \frac{M \bar \eta^4 \bar \mu}{{c^*}^4}} + (\sigma^* + c^*) \sqrt{KT}\right).
\end{align*} 
proving the final claim of \cref{thm:regfin}. 

\subsection{Regret Proof in `Low-Noise Regime'}
\label{pf:regfin2}

\thmregfintwo*

\begin{proof}
To analyze the regret of this regime, we begin by deriving a lemma for the mean reward concentration similar to \cref{lem:muConc}.
\begin{restatable}[Mean-Reward Concentration: Case 2]{lemma}{muConctwo}
\label{lem:muConctwo}
Consider any $\delta \in (0,1)$. If there exists no $\ j $ such that $ \sigma_j^2 \geq {c^*}^2$ then for any time step $t \in [T], i \in [K]$, and any realization of $\alpha > {\bar \eta}^2 c^* \log(3KT/\delta)$, with probability at least $1-\dfrac{\delta}{3}$: 
    \begin{align*}
        | \mu_i - \hat \mu_t(i) | 
        & \leq \frac{ 2{c^*}\sqrt{\log(3KT/\delta)}}{\sqrt{n_i(t)}}.
    \end{align*}
\end{restatable}

\begin{proof}
Suppose $T_t(i) = \{s \leq t : i_s = i\}$ denotes the set of time steps at which arm $i$ was selected up to time $t$. Consider the sequence of random variables defined by
\[
D_{t,i} \defeq X_{t,i} - \mu_i, \quad \text{for all } t \in T_t(i).
\]

To establish a concentration bound, the key observation is that $\{D_{t,i}\}_{t \in T_t(i)}$ forms a martingale difference sequence with respect to the filtration 
$
\mathcal{F}_{t-1} = \sigma\left(\{X_s, i_s, j_s\}_{s=1}^{t-1}, i_t = i\right),
$
generated by all arm selections and observations prior to round $t$. Precisely note that each $D_{t,i}$ is integrable, adapted to $\mathcal{F}_{t}$, and satisfies
$
\E[D_{t,i}\mid\mathcal{F}_{t-1}] 
= \E[X_{t,i} - \mu_i \mid \mathcal{F}_{t-1}]
= \mu_i - \mu_i = 0.
$
Thus, by construction, the sequence satisfies the martingale difference property and enables us to leverage standard martingale concentration inequalities.

\freedmanineq*

We have a bound $\bar \eta$ on $D_{t,i}$, and the first requirement of the theorem has already been proven. Applying Freedman's, we get

\begin{align*}
&\mathbb{P}\left(\sum_{t \in T_i(t)} D_{t,i} \leq \lambda \sum_{t \in T_i(t)} \mathbb{E}[D_{t,i}^2 | \mathcal{F}_{t-1}] + \frac{\log(3/\delta)}{\lambda}\right) \geq 1 - \dfrac{\delta}{3},
\\
&\mathbb{P}\left(\sum_{t \in T_i(t)} D_{t,i} \leq \lambda \sum_{j=1}^M n_{ij}(t)\cdot\sigma_{j}^2 + \frac{\log(3/\delta)}{\lambda}\right) \geq 1 - \dfrac{\delta}{3}. ~~\left[\mathbb{E}[D_{t,i}^2 | \mathcal{F}_{t-1}] = \sigma_{j}^2 \right] 
\\
\end{align*}
 Choosing $\lambda^* = \dfrac{\sqrt{\log(3/\delta)}}{\sqrt{{c^*}^2\cdot n_{i}(t)}}$. 
\begin{align*}
\mathbb{P}\left(\sum_{t \in T_i(t)} D_{t,i} \leq \dfrac{ \sqrt{\log(3/\delta)} \sum_{j=1}^M n_{ij}(t)\cdot\sigma_{j}^2}{\sqrt{{c^*}^2 \cdot n_{i}(t)}} + \dfrac{\sqrt{{c^*}^2 \cdot n_{i}(t)}}{\sqrt{\log (3/\delta)}}\cdot\log(3/\delta)\right) &\geq 1 - \dfrac{\delta}{3},
\\
\mathbb{P}\left(\sum_{t \in T_i(t)} D_{t,i} \leq \dfrac{\sqrt{\log(3/\delta)} \sum_{j=1}^M n_{ij}(t)\cdot {c^*}}{\sqrt{n_{i}(t)}} + c^*\sqrt{n_{i}(t)\cdot\log(3/\delta)}\right) &\geq 1 - \dfrac{\delta}{3}, ~~[\text{$\sigma_j^2 < c^*$}]
\\
\mathbb{P}\left(\sum_{t \in T_i(t)} D_{t,i} \leq  2c^*\sqrt{n_{i}(t) \log(3/\delta)}\right) &\geq 1 - \dfrac{\delta}{3}. 
\end{align*}
Dividing both sides by $n_{i}(t)$, union bounding over all arms and across all timesteps, we get our mean concentration bound w.p $1-\dfrac{\delta}{3}$
\[
\abs{\hat{\mu}_i(t) - \mu_i} \leq \dfrac{2c^*\sqrt{\log(3KT/\delta)}}{\sqrt{n_{i}(t)}}.
\]

Now to satisfy the Freedman constraint we need:
\begin{align*}
    \lambda^* = \dfrac{ c^* \sqrt{\log(3KT/\delta)}}{\sqrt{n_{i}(t)}} &< \dfrac{1}{\bar \eta},
    \\
    n_{t}(i) &> \bar \eta^2 {c^*}^2 \log(3KT/\delta).
\end{align*}
That is, each arm $i \in [K]$ is pulled at least
$
\alpha = \dfrac{\bar{\eta}^2 \log(3KT/\delta)}{{c^*}^2}
$
times to ensure that the mean concentration bound and the subsequently derived UCB estimate \cref{eq:ucb-mu-two} hold with high probability.\\
\end{proof}

\paragraph{UCB equation:}
The best arm can be identified via an Upper Confidence Bound estimate using the derived mean reward concentration \cref{lem:muConctwo}, specifically
\begin{equation}\label{eq:ucb-mu-two}
\ucb_t^\mu(i) \defeq \hat \mu_t(i) + \frac{  2 {c^*}\sqrt{\log(3KT/\delta)}}{\sqrt{n_t(i)}}.
\end{equation}
Note that under our regime, we may treat every surviving source as having effective variance $c^*$ and then apply a standard UCB multi-armed bandit analysis using \cref{eq:ucb-mu-two}. With that in mind, let us returning back to the proof of \cref{thm:regfintwo}, where the regret can be derived as follows:
\paragraph{Bounding $n_{i}(t)$}
Note that at any time $t$, arm-$i$ can not be selected if $\ucb_t^\mu(i) \leq \ucb_t^\mu(i^*)$. So arm-$i$ only gets selected at time $t$ if:\\
\begin{align*}
 \mu_{i^*} & \leq \ucb_t^\mu(i^*) \leq \ucb_t^\mu(i) = \hat{\mu}_t(i) + \dfrac{2{c^*}\sqrt{ \log(3KT/\delta)}}{\sqrt{n_{i}(t)}}, 
\\
\mu_{i^*} &  \leq {\mu}_t(i) + \dfrac{4{c^*}\sqrt{ \log(3KT/\delta)}}{\sqrt{n_{i}(t)}},  ~~[\text{Using \cref{lem:muConctwo}} ]
\\
\implies n_{i}(t) & \leq \dfrac{16{c^*}^2 \log(3KT/\delta)  }{\Delta_i^2}. 
\end{align*}

\paragraph{Total Regret :}
We begin by deriving the regret incurred from the exploration phase.
\begin{align*}
\text{Reg}^{\text{Exp}}_T &= M\tau_p \bar \mu + K\alpha \bar \mu  
\\
 &=   {M \left(\dfrac{1024 \bar \eta^4 \log(12M/\delta)}{{c^*}^4} \right) \bar \mu} + K\left( \bar \eta^2 {c^*}^2\log(3KT/\delta)\right) \bar \mu  \space.   
\end{align*}

The Regret in $T - K\alpha $ rounds of Running \algmab \ in this regime :
\begin{align*}
    \text{Reg}_T  = \sum_{i = 2}^K n_i(T) \Delta_i &\leq  \sum_{i=2}^K \dfrac{16{c^*}^2 \log(3KT/\delta)  }{\Delta_i}, 
    \\
    &\leq 16{c^*}^2 \log(3KT/\delta) \sum_{i=2}^K \dfrac{1}{\Delta_i},
    \\
    & \leq  \dfrac{{c^*}^2L^2}{6} \sum_{i=2}^K \dfrac{1}{\Delta_i}.
\end{align*}
where $L = 4\sqrt{\log(3KT/\delta)}$

The total regret of \algmab $(\text{Reg}^{\text{\algmab}}_T)$  \ = \  $\text{Reg}^{\text{Exp}}_T+\text{Reg}_T = M\tau_p \bar \mu \ + \  K\alpha \bar \mu \  \ +  \sum_{i = 2}^K \dfrac{ {c^*}^{2}L^2}{6\Delta_i} $ \\
where
\begin{enumerate}
    \item $\tau_p = \dfrac{1024 \bar \eta^4 \log(12M/\delta)}{{c^*}^4}$. 
    \item $\alpha = \bar \eta^2 {c^*}^2\log(3KT/\delta)$. 
\end{enumerate}

\paragraph{Order of Regret:}
\begin{align*}
\tilde O(\text{Reg}^{\text{\algmab}}_T) =& \tilde O\left( \dfrac{ M\bar \eta^4 \bar \mu}{{c^*}^4} + K\bar \eta^2 {c^*}^2 \bar \mu +  \sum_{i = 2}^K \dfrac{ {c^*}^{2}}{\Delta_i}  \right). \\ 
\end{align*}

\paragraph{Worst Case Bound:}
This proves the first part of \cref{thm:regfintwo}, which yields an instance-dependent guarantee of \algmab.
\noindent Finally, to see the worst case (instance independent) regret analysis of \algmab, note that the regret in $T-M\tau_p - K\alpha - \tilde M\beta$ rounds of \algmab\, can be alternatively bounded as: 
\begin{align*}
\text{Reg}_T \leq \dfrac{{c^*}^2L^2 (K-1) \sqrt{T}}{6}.   
\end{align*}

Therefore, the order of  the instance-independent regret is
\begin{align*}
&\tilde O(\text{Reg}_T^{\text{Exp}} +  K{c^*}^2\sqrt{T}\\
\equiv \ &\tilde O\left( \dfrac{ M\bar \eta^4 \bar \mu}{{c^*}^4} + {K \bar \eta^2 {c^*}^2 \bar \mu} + K{c^*}^2\sqrt{T}\right).
\end{align*} 
proving the final claim of \cref{thm:regfintwo}
\end{proof}
\section{Experimental Setup and Extended Results}
\label{sec:app_expts}
\
In this section, we describe the experimental setup and present our baseline comparison results along with the results obtained by varying the number of arms $K$ and the number of sources $M$.

\subsection{Variation in the Number of Arms and Sources}
\begin{figure}[H]
    \centering
    \begin{minipage}[t]{0.48\linewidth}
        \centering
        \includegraphics[width=\linewidth, height=4cm]{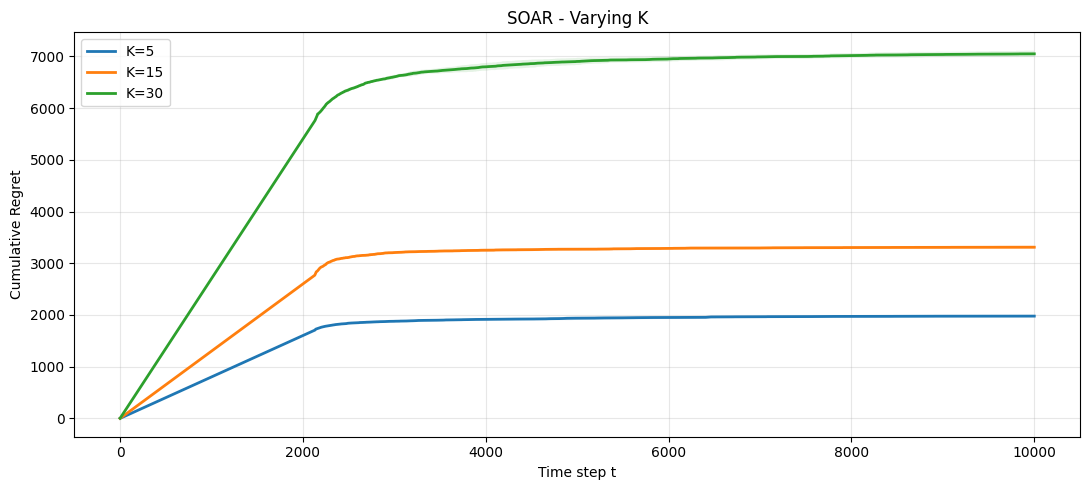}
        \caption{Regret of \algmab\ with varying number of arms $K \in \{5,15,30\}$}
        \label{fig:varyingK}
    \end{minipage}
    \hfill
    \begin{minipage}[t]{0.48\linewidth}
        \centering
        \includegraphics[width=\linewidth, height=4cm]{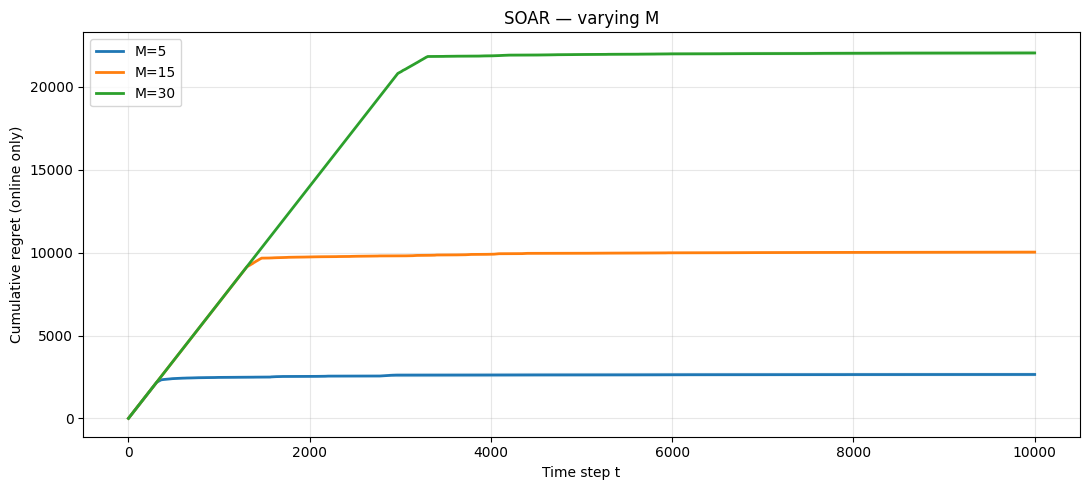}
        \caption{Regret of \algmab\ with varying number of sources $M \in \{5,15,30\}$}
        \label{fig:varyingM}
    \end{minipage}
\end{figure}

We evaluate \algmab\ on synthetic multi-source bandit tasks under two complementary setups.  
Arm rewards are modeled as Gaussian random variables with fixed means and source-dependent variances. 
Arm means are either drawn uniformly from $[1,10]$ (rounded to one decimal place) or fixed as $\mu = [1,5,8,6,4]$ as described below. 
Source variances are specified directly or sampled uniformly from $[1,3]$ (rounded to one decimal place). We can choose $\nu$ as per \cref{rem:regfin} and $c^* = 2$.
All runs are executed for a time horizon of $T=10{,}000$.  
In both setups, we observe an initial linear growth in regret, which stems from the preprocessing phase combined with the initial exploration rounds. 

\paragraph{Varying $K$} 
We fix the number of sources to $M=3$ with variances 
$\sigma = [5,1,10]$, and vary the number of arms $K \in \{5,15,30\}$. 
For each configuration, arm means are drawn uniformly from $[1,10]$ (rounded to one decimal place). 
This setup isolates how \algmab\ scales with the number of arms while holding source variability constant.  

\paragraph{Varying $M$} 
We fix the arms to $\mu = [1,5,8,6,4]$ and vary the number of sources $M \in \{5,15,30\}$. 
The source variances are sampled independently from $[1,3]$ (rounded to one decimal place). 
This setup isolates how \algmab\ scales with the number of sources while holding the arm structure fixed.  The initial growth in regret clearly reflects the expected scaling with $M$, consistent with our theoretical analysis.

\subsection{Baseline Comparison Results}
\label{sec:exptsBaseline}
\as{put this back to main expts section for arXiv}

For both baselines, we set arm means in the range $[0,1]$ and source variances in the range $[1,10]$. Both scenarios are evaluated over $T=20{,}000$ rounds. 
\begin{figure}[H]
    \centering
    \begin{minipage}[t]{0.48\linewidth}
        \centering
        \includegraphics[width=\linewidth, height = 4cm]{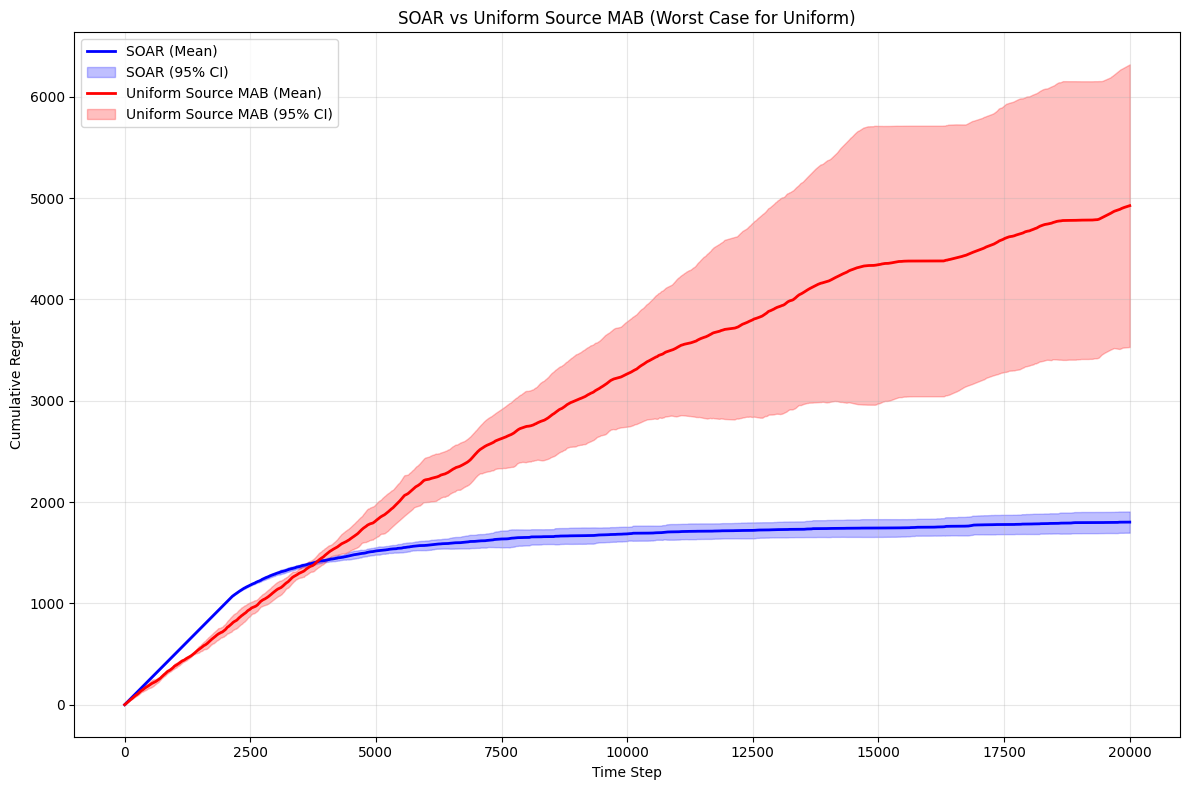}
        
        \caption{\algmab \space vs Baseline 1: WC1}
        \label{fig:varyingK}
    \end{minipage}
    \hfill
    \begin{minipage}[t]{0.48\linewidth}
        \centering
        \includegraphics[width=\linewidth, height = 4cm]{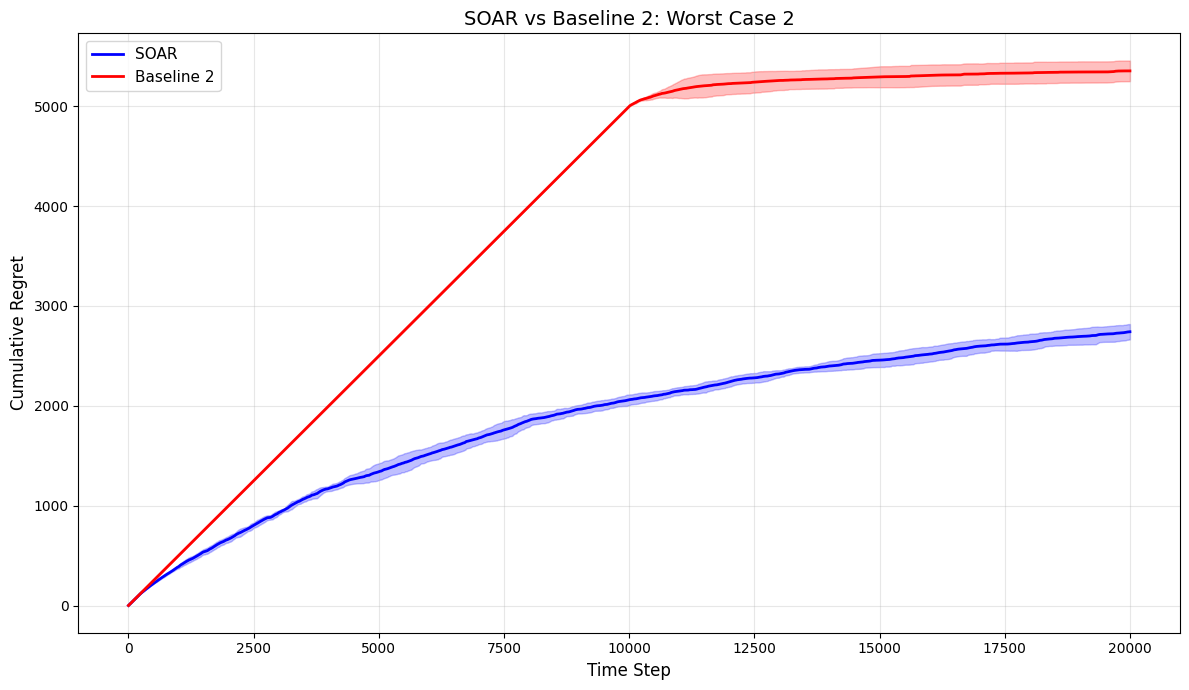}
        
        \caption{\algmab\space vs Baseline 2: WC2}
        \label{fig:varyingM}
    \end{minipage}
\end{figure}
\paragraph{Baseline 1: Uniform Source MAB}For the uniform-source baseline (WC-1: $K=5$, $M=3$), we construct instances with multiple high-variance sources alongside a low-variance source, creating a stark variance disparity.
In this setting, \algmab\ initially incurs a higher cost due to its exploration phase, but then stabilizes and achieves substantially lower regret by adaptively prioritizing the low-variance source. 
In contrast, the uniform baseline continues to suffer from repeatedly sampling the high-variance sources throughout the horizon.
 
\paragraph{Baseline 2: Two Phase MAB}For the two-phase baseline (WC-2: $K=10$, $M=8$), we instead consider sources with gradually increasing variances, introducing only incremental differences across sources. 
Here, \algmab\ effectively handles the fine-grained variance differences by relying on continuous confidence bounds for adaptive source selection, whereas the two-phase baseline incurs significant regret due to its rigid elimination phase.

\subsection{MovieLens Panel}
We evaluate source adaptivity on a real-world ratings dataset using a fixed panel constructed from the MovieLens-25M dataset \cite{harper2015movielens}.

\begin{figure}[H]
    \centering
    \begin{minipage}[t]{0.48\linewidth}
        \centering
        \includegraphics[width=\linewidth, height=4cm]{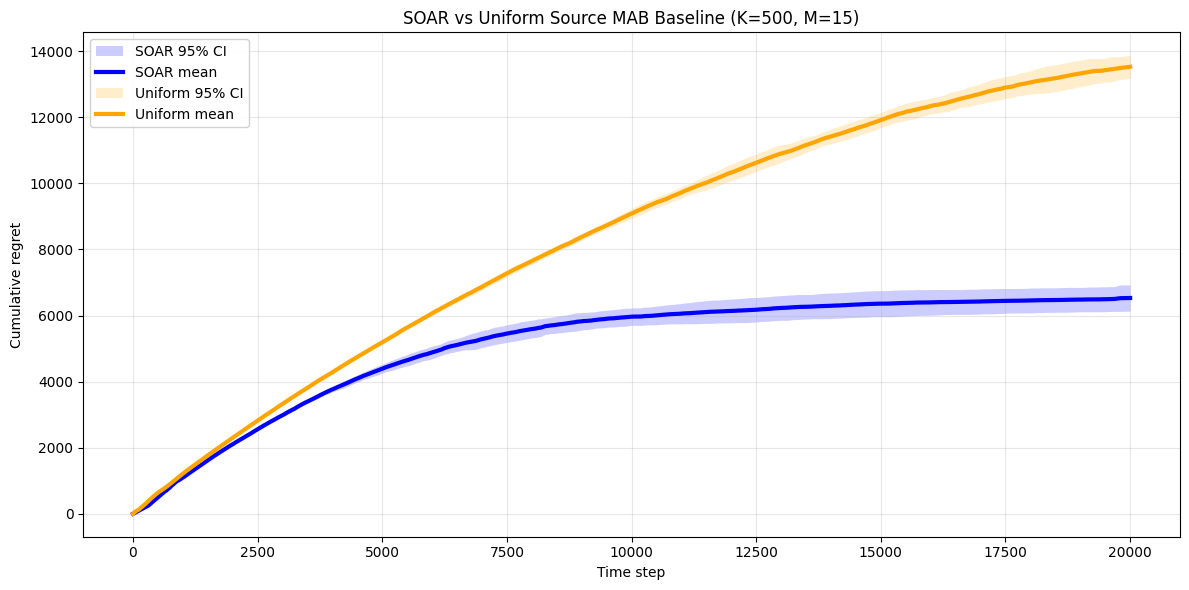}
        \caption{\algmab\ vs.\ Baseline-1 on MovieLens.}
        \label{fig:ml_b1}
    \end{minipage}\hfill
    \begin{minipage}[t]{0.48\linewidth}
        \centering
        \includegraphics[width=\linewidth, height=4cm]{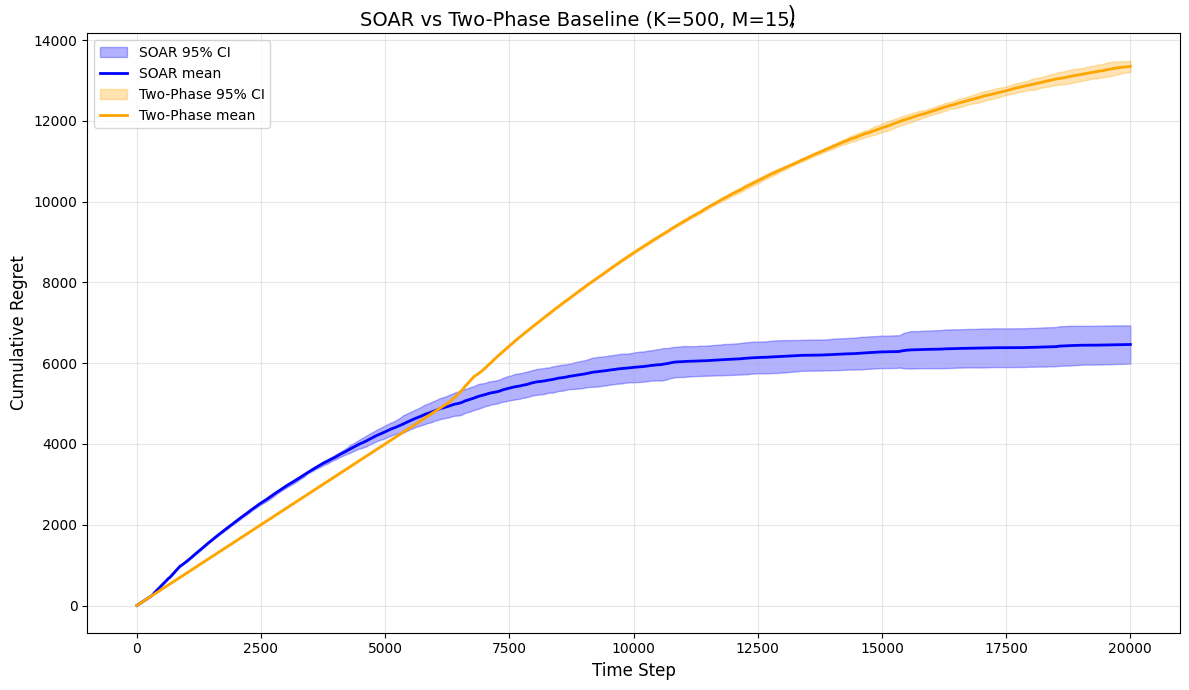}
        \caption{\algmab\ vs.\ Baseline-2 on MovieLens.}
        \label{fig:ml_b2}
    \end{minipage}
\end{figure}

We select $M=15$ reviewers who have each rated the same set of $K=500$ movies, yielding a total of $7{,}500$ ratings.
Each movie is treated as an arm, with its unknown mean reward $\mu_i \in [0,5]$ corresponding to the average rating within the panel. 
Each reviewer is treated as a data source, and the empirical variance of their ratings over the selected movies is used as the source variance. Hence $\bar \eta^2 = 25$. We choose $c^* = 1, \nu = 30$.  At each round, the learner selects a movie--reviewer pair and observes the corresponding rating.
Regret is computed with respect to the movie with the highest empirical mean rating in the panel.
We compare \algmab\ against two natural baselines: (i) a uniform-source baseline that samples reviewers uniformly at random, and (ii) a two-phase baseline that first identifies a low-variance reviewer via uniform exploration and then runs a standard UCB algorithm using only that reviewer.
All algorithms are run for a fixed horizon ($T = 20{,}000$), and results are averaged over multiple independent runs with shaded regions indicating $95\%$ confidence intervals.

\cref{fig:ml_b1} and \cref{fig:ml_b2} report cumulative regret as a function of time for \algmab\ and the two baselines. The figures show that \algmab\ consistently incurs lower cumulative regret than both baselines by rapidly concentrating sampling on low-variance reviewers while identifying highly rated movies. These results indicate that \algmab\ remains effective on real-world rating data, and not only on controlled synthetic benchmarks.


\end{document}